\documentclass[10pt]{article}

\usepackage[T1]{fontenc}
\usepackage{microtype}
\usepackage[margin=1in]{geometry}
\usepackage{booktabs}
\usepackage{amsmath,amssymb,amsfonts,amsthm}
\usepackage{bm}
\usepackage{enumitem}
\usepackage{graphicx}
\usepackage{caption}
\usepackage{subcaption}
\usepackage[export]{adjustbox}
\usepackage{placeins}
\usepackage{algorithm}
\usepackage{algpseudocode}
\usepackage{multirow}
\usepackage[hidelinks]{hyperref}

\newcommand{\Ccal}{\mathcal{C}}
\newcommand{\Acal}{\mathcal{A}}
\newcommand{\Ucal}{\mathcal{U}}

\newcommand{\argmin}{\operatorname*{arg\,min}}

\theoremstyle{plain}

\newtheorem{proposition}{Proposition}
\newtheorem{lemma}{Lemma}
\newtheorem{corollary}{Corollary}

\theoremstyle{definition}

\newtheorem{assumption}{Assumption}

\theoremstyle{remark}
\newtheorem{remark}{Remark}

\usepackage[backend=biber,style=numeric-comp,sorting=none,maxbibnames=99,doi=true,url=false,isbn=false,eprint=true]{biblatex}
\title{Reference-free logged energy-oracle recovery for neural approximations
of symmetric coercive variational problems: conforming Riesz reconstruction
and archive-level selection}

\author{%
Karim Bounja\textsuperscript{1,*}\quad
Lahcen Laayouni\textsuperscript{2}\quad
Boujemaa Achchab\textsuperscript{1}\quad
Abdeljalil Sakat\textsuperscript{1}
}

\date{}

\hypersetup{
  pdftitle={Reference-free logged energy-oracle recovery for neural approximations of symmetric coercive variational problems: conforming Riesz reconstruction and archive-level selection},
  pdfauthor={Karim Bounja, Lahcen Laayouni, Boujemaa Achchab, Abdeljalil Sakat},
  pdfkeywords={Reference-free checkpoint selection, Neural PDE solvers, Symmetric coercive variational problems, Conforming Riesz reconstruction, Energy-norm error estimation, Oracle selection consistency}
}

\begin{document}

\maketitle

\begingroup
\renewcommand{\thefootnote}{\arabic{footnote}}
\footnotetext[1]{Laboratory for Analysis and Modeling of Systems and Decision Support (LAMSAD), Hassan 1st University of Settat, Settat 26000, Morocco.}
\footnotetext[2]{Department of Computer Science, School of Science and Engineering, Al Akhawayn University in Ifrane, Ifrane 53000, Morocco.}
\renewcommand{\thefootnote}{*}
\footnotetext[3]{Corresponding author: \href{mailto:k.bounja.doc@uhp.ac.ma}{\texttt{k.bounja.doc@uhp.ac.ma}}.}
\endgroup

\begin{abstract}
Neural PDE training yields a finite checkpoint archive, yet its logged energy
errors are inaccessible without the exact solution, while loss-based selection
does not necessarily recover the logged energy oracle. For admissible neural
approximations of symmetric coercive variational problems, we introduce a
reference-free selection rule based on minimizing a computable conforming
Riesz monitor. The exact residual--energy identity and conforming projection make the monitor
an unconditional lower bound converging monotonically to each logged energy
error under nested conforming refinement; under saturation, hierarchical
enrichment yields a computable upper estimate and hence a lower--upper bracket.
A key finding is that archive selection is order-sensitive: unresolved
checkpoint-dependent components can reverse the oracle--non-oracle ranking at
finite resolution, so checkpointwise recovery alone is insufficient. For finite archives, we prove uniform recovery, yielding convergence to the
logged-oracle error and, without saturation, logged-oracle selection at
sufficiently fine auxiliary resolution. Under saturation, the bracket gives a
computable near-oracle bound and certifies unique logged-oracle selection upon
interval separation. We also bound logging-resolution loss and certify oracle
inclusion over prescribed comparison trajectories. The resulting criterion replaces inaccessible exact-error minimization by
computable, training-independent post-training selection on the intrinsic
energy-error scale, requiring only the computed candidates and the variational
problem. Experiments on diffusion and elasticity, including a non-manufactured
perforated plate, demonstrate energy-scale calibration, oracle-level selection,
and modest post-processing cost. This work frames neural checkpoint selection as an ordering problem, establishes projection-based reliability guarantees beyond checkpointwise error estimation and loss--error equivalence, and opens directions for variational loss design and more general selection theory.
\end{abstract}

\noindent\textbf{Keywords:} Reference-free checkpoint selection; Neural PDE solvers; Symmetric coercive variational problems; Conforming Riesz reconstruction; Energy-norm error estimation; Oracle selection consistency.

\section{Introduction}
\label{sec:introduction}

Neural solvers are increasingly used to approximate solutions of partial
differential equations, yet their practical reliability depends on assessing
the computed output in a norm intrinsic to the underlying problem
\cite{DeRyck2024,ZeinhoferMasriMardal2025}. When training yields a finite archive
\(\{u_{\theta_k}\}_{k\in\mathcal K_{\mathrm{log}}}\), a recorded output is
typically retained as the numerical solution. Let
\((\mathcal X,\|\cdot\|_{\mathcal X})\) denote the problem-dependent solution
space and the norm in which approximation errors are to be assessed. The
reference selection criterion is then to retain the available approximation
closest to the exact solution, that is, to identify
\[
\operatorname*{arg\,min}_{k\in\mathcal K_{\mathrm{log}}}
\|u_{\theta_k}-u^\ast\|_{\mathcal X}.
\]
This ideal error criterion cannot be evaluated directly because the exact
solution \(u^\ast\) is unknown. Accordingly, checkpoint selection on the
problem-intrinsic error scale calls for computable diagnostics aligned with
the same intrinsic geometry.

For the symmetric coercive problems considered here, let
\(\mathcal A:=u_b+V_0\) denote the affine admissible class, where \(V_0\)
is the space of homogeneous admissible directions. Since the exact solution
\(u^\ast\) and every admissible recorded approximation \(u_{\theta_k}\)
belong to \(\mathcal A\), \(u_{\theta_k}-u^\ast\in V_0\). The symmetric
coercive bilinear form \(a\) induces the energy norm
\(\|v\|_a:=a(v,v)^{1/2}\) on \(V_0\). For the finite checkpoint archive, define
\[
E_k:=\|u_{\theta_k}-u^\ast\|_a,
\qquad
\mathcal K_E^\ast
:=
\operatorname*{arg\,min}_{k\in\mathcal K_{\mathrm{log}}}E_k.
\]
Although the errors \(E_k\) are inaccessible, the quadratic Ritz functional
\(J(u):=\frac12 a(u,u)-\ell(u)\) and the variational residual
\[
\langle R(u),v\rangle:=\ell(v)-a(u,v),
\qquad v\in V_0,
\]
satisfy, for every logged checkpoint,
\[
J(u_{\theta_k})-J(u^\ast)
=
\frac12 E_k^2
=
\frac12\|R(u_{\theta_k})\|_{V_{0,a}'}^2.
\]
Thus, the energy-dual residual is exactly calibrated with the target energy
error.

For a finite-dimensional conforming auxiliary space \(V_m\subset V_0\), let
\(z_{m,k}\in V_m\) satisfy
\[
a(z_{m,k},v)
=
\langle R(u_{\theta_k}),v\rangle,
\qquad
v\in V_m,
\]
and define the conforming Riesz monitor
\(\eta_{m,k}:=\|z_{m,k}\|_a\). Conforming projection gives
\[
\eta_{m,k}\leq E_k,
\qquad
\eta_{m,k}\uparrow E_k
\]
along a nested conforming auxiliary hierarchy with dense union in \(V_0\).
This checkpointwise recovery naturally suggests the computable selection rule
\[
\mathcal K_m^\ast
:=
\operatorname*{arg\,min}_{k\in\mathcal K_{\mathrm{log}}}
\eta_{m,k}.
\]

The nontrivial issue is whether checkpointwise recovery lifts to consistent
selection over the archive. Since selection is an ordering problem, oracle
recovery requires preservation of the oracle--non-oracle separation induced
by the exact energy errors. At a fixed auxiliary resolution, this need not
hold: checkpoint-dependent unresolved components may reverse the
oracle--non-oracle ranking. We refer to this as the finite-resolution
ordering obstruction.

For a finite archive, we show that checkpointwise recovery upgrades to the
uniform estimate
\[
\varepsilon_m
:=
\max_{k\in\mathcal K_{\mathrm{log}}}
|E_k-\eta_{m,k}|
\longrightarrow0.
\]
This uniform recovery eventually restores the oracle--non-oracle separation.
Consequently, every \(k_m^\ast\in\mathcal K_m^\ast\) satisfies
\[
0
\leq
E_{k_m^\ast}-E_\ast
\leq
\varepsilon_m,
\qquad
E_\ast:=\min_{k\in\mathcal K_{\mathrm{log}}}E_k,
\]
and, for all sufficiently large \(m\),
\[
\mathcal K_m^\ast
\subseteq
\mathcal K_E^\ast.
\]
Thus, monitor minimization is consistent in value and eventually recovers the
logged energy oracle without any saturation assumption.

At finite auxiliary resolution, nested conforming reconstructions provide,
under saturation, computable near-oracle bounds and exact logged-oracle
certification by interval separation. We further separate auxiliary resolution
from logging resolution: relative to a prescribed finite comparison
trajectory, competitive-coverage bounds quantify the loss induced by archive
subsampling and yield certificates of trajectory-oracle inclusion and unique
recovery. The resulting framework replaces inaccessible exact-error
minimization by a computable, training-independent post-training selection
procedure requiring neither the exact solution nor a reference solve.
For archive selection, the required auxiliary accuracy is set by the
oracle--non-oracle separation: exact recovery of every individual checkpoint
error is not required.

The present work lies at the intersection of two lines of analysis:
error estimates relating residual or variational training objectives to
solution errors under suitable stability assumptions
\cite{ShinZhangKarniadakis2023,ZeinhoferMasriMardal2025},
and a posteriori constructions providing computable error indicators for
neural PDE approximations
\cite{MinakowskiRichter2023,BerroneCanutoPintore2022}.

For training objectives intended to provide quantitative control of the
error with respect to the exact solution, recent work has identified
calibration with the problem-intrinsic error as a structural design
requirement. Consistent collocation losses providing error control in the
target energy norm have been constructed
\cite{BonitoDeVorePetrovaSiegel2026}, while variationally correct residual
losses whose values are uniformly equivalent to the squared solution error
in the norm determined by a stable variational formulation have been
developed in \cite{BachmayrDahmenOster2025}. Related strategies have also
incorporated computable residual-based error information directly into the
training objective, including discrete dual residual norms and a posteriori
estimators
\cite{RojasMaczugaMunozMatutePardoPaszynski2024,FuhrerRojas2025}.

For archive selection, however, even uniform checkpointwise loss--error
equivalence does not by itself guarantee preservation of a minimum-error
checkpoint under loss minimization over a finite archive. Indeed, suppose
that, for some \(p>0\) and constants \(0<c<C\),
\[
cE_k^p
\leq
\mathcal L_k
\leq
CE_k^p,
\qquad
k\in\mathcal K_{\mathrm{log}}.
\]
If two checkpoints satisfy \(E_i<E_j\) while \(CE_i^p>cE_j^p\), these bounds
are compatible with the admissible values
\(\mathcal L_i=CE_i^p\) and \(\mathcal L_j=cE_j^p\), for which
\(\mathcal L_i>\mathcal L_j\). Thus, uniform loss--error equivalence controls
each candidate checkpointwise but does not, by itself, imply preservation of
the target-error ordering or guarantee that every loss minimizer over the
archive is a target-error minimizer.

Independent post-training error assessment offers a complementary route:
rather than inferring target accuracy from the training loss, a posteriori
approaches provide computable information on the error of a computed
approximation. Representative constructions include a dual-weighted-residual
estimator for Deep Ritz approximations of the Laplace and Stokes problems
\cite{MinakowskiRichter2023}, a reliable and efficient energy-norm estimator
for VPINNs combining residual, loss, and data-oscillation contributions
\cite{BerroneCanutoPintore2022}, and rigorous lower and upper error bounds
obtained from Riesz representations of residual extensions and restrictions
to geometrically simpler domains \cite{ErnstRekatsinasUrban2025}. A related line constructs neural approximation spaces sequentially, with the
resulting updates providing error indicators and stopping criteria in Galerkin
and collocation settings \cite{AinsworthDong2021,WengMaoShen2026}; an extended
Galerkin framework addresses general boundary-value problems with singular
solution structures \cite{AinsworthDong2025}. For general
well-posed variational formulations, checkpointwise error control has been
obtained through a computable projected-residual contribution together with a
complementary residual contribution admitting reliable a posteriori control
\cite{FuhrerRojas2025}. These approaches provide checkpointwise error
assessment. The finite-resolution ordering obstruction for the conforming
Riesz monitor, analyzed in
Proposition~\ref{prop:finite-resolution-rank-reversal}, however, shows that
checkpointwise control alone does not ensure reliable archive-level selection.
More generally, target-error order preservation is not guaranteed beyond this
setting; see Section~\ref{sec:discussion}.

The present work therefore makes finite-archive checkpoint selection itself
the object of analysis in the symmetric coercive setting, identifying the
finite-resolution ordering obstruction and establishing its removal through
uniform recovery over the archive. The resulting reference-free rule is
\[
\mathcal K_m^\ast
:=
\operatorname*{arg\,min}_{k\in\mathcal K_{\mathrm{log}}}\eta_{m,k}.
\]
For finite archives, conforming auxiliary refinement yields eventual
logged-oracle selection; under saturation, nested reconstructions provide
finite-resolution near-oracle bounds and interval-separation certificates,
while logging-resolution control quantifies the loss relative to a prescribed
finite comparison trajectory and yields trajectory-oracle certificates. Thus,
once optimization has generated sufficiently accurate admissible
candidates---the exact Ritz objective being energy-calibrated by
Lemma~\ref{lem:energy_gap}---monitor refinement recovers the logged oracle,
while logging-resolution control determines whether it is also an oracle of
the optimization trajectory.

To the best of our knowledge, this passage from checkpointwise error
assessment to reference-free archive-level oracle selection has not previously
been analyzed for neural PDE approximations. The framework enables reference-free recovery and certification of
energy-optimal checkpoints within the retained archive and relative to a
prescribed finite comparison trajectory.

For scientific machine learning, this work provides
both a conceptual and a methodological contribution.
It formulates neural checkpoint selection as an
order-preservation problem, distinct from
checkpointwise error estimation, and develops a
training-independent procedure with oracle-recovery
and conditional certification guarantees. This
perspective also clarifies the role of variational
training objectives: exact Ritz values preserve
energy-error ordering, whereas numerical perturbations
may alter it, and loss--error equivalence alone does
not guarantee reliable selection. These distinctions
provide a basis for investigating selection reliability
under more general variational losses.

The numerical study tests calibration, auxiliary and logging refinement, and
oracle selection on manufactured diffusion and elasticity benchmarks and on a
non-manufactured perforated plate assessed against an independent
refined FEM reference. The results show energy-error calibration at
sufficient auxiliary resolution, oracle-level selection, controlled logging
loss, and low monitoring overhead.

The remainder of the paper is organized as follows.
Sections~\ref{sec:setting}--\ref{sec:trajectory_wide_control} develop the
variational setting, conforming Riesz reconstruction, logged-oracle guarantees,
and logging-resolution control. Section~\ref{sec:numerical} presents the
numerical validation, and Section~\ref{sec:discussion} discusses broader
implications for archive-level selection.

\section{Problem setting}
\label{sec:setting}

We consider symmetric coercive variational problems posed on an affine
admissible class and approximated by parametrized functions belonging to that
same class.

\subsection{Abstract symmetric coercive setting}
\label{subsec:abstract_setting}

Let \(X\) be a real Hilbert space and \(V_0\subset X\) a closed linear
subspace representing the homogeneous admissible directions associated with
the essential constraints. Given a lifting \(u_b\in X\) of the prescribed
essential data, define
\[
\Acal:=u_b+V_0.
\]
Thus, admissible functions belong to \(\Acal\), while differences of
admissible functions belong to \(V_0\). For homogeneous essential data,
\(u_b=0\) and \(\Acal=V_0\).

Let \(a:X\times X\to\mathbb R\) be a continuous symmetric bilinear form whose
restriction to \(V_0\times V_0\) is coercive. Thus, there exist constants
\(0<\alpha_a\le\beta_a<\infty\) such that
\[
\begin{aligned}
|a(w,v)|
&\le
\beta_a\|w\|_X\|v\|_X,
&&\forall w,v\in X,\\
a(v,v)
&\ge
\alpha_a\|v\|_X^2,
&&\forall v\in V_0.
\end{aligned}
\]
Let \(\ell\in X'\). The exact solution is the unique
\(u^\ast\in\Acal\) satisfying
\begin{equation}
a(u^\ast,v)=\ell(v),
\qquad
\forall v\in V_0.
\label{eq:abstract_variational_problem}
\end{equation}
Indeed, subtracting the lifting \(u_b\) reduces the affine problem to a
coercive problem on \(V_0\), to which the Lax--Milgram theorem applies.

The bilinear form \(a\) induces the energy inner product and norm on \(V_0\),
\[
(w,v)_a:=a(w,v),
\qquad
\|v\|_a^2:=a(v,v),
\qquad
w,v\in V_0.
\]
For \(u\in X\), define the Ritz functional
\(J(u):=\frac12 a(u,u)-\ell(u)\).
Its minimization over \(\Acal\) has
\eqref{eq:abstract_variational_problem} as its first-order optimality
condition.

Let \(V_0'\) denote the topological dual of \(V_0\). For every
\(u\in\Acal\), the map
\[
v\mapsto \ell(v)-a(u,v)
\]
is continuous on \(V_0\). We therefore define the variational residual
\(R(u)\in V_0'\) by
\[
\langle R(u),v\rangle
:=
\ell(v)-a(u,v),
\qquad
\forall v\in V_0,
\]
and measure it in the dual norm induced by the energy norm,
\[
\|F\|_{V_{0,a}'}
:=
\sup_{v\in V_0\setminus\{0\}}
\frac{|\langle F,v\rangle|}{\|v\|_a},
\qquad
F\in V_0'.
\]

The residual--energy identities of
Section~\ref{sec:riesz_error_assessment} are formulated on \(V_0\), while
the Riesz reconstructions and hierarchical estimators of
Subsection~\ref{subsec:conforming_riesz_hierarchy} are defined on conforming
subspaces of \(V_0\). The key admissibility property is
\[
u-u^\ast\in V_0,
\qquad
u,u^\ast\in\Acal,
\]
which allows the energy-dual norm of the variational residual to be
identified with the energy error.

\subsection{Admissible neural Ritz approximations}
\label{subsec:admissible_neural_ritz}

A parametrized approximation \(u_\theta\) is called admissible if
\(u_\theta\in\Acal\), and therefore \(u_\theta-u^\ast\in V_0\), as required
by the residual--energy identity.

For concrete boundary-value problems, let \(\Gamma_D\) denote the portion of
the boundary carrying Dirichlet-type essential conditions. Admissibility may
be enforced through a hard boundary ansatz chosen so that its homogeneous
correction belongs to \(V_0\). For homogeneous essential data, one may use
\begin{equation}
u_\theta(x)
=
g(x)N_\theta(x),
\qquad
g=0
\quad
\text{on } \Gamma_D,
\label{eq:hard_boundary_ansatz_homogeneous}
\end{equation}
where \(N_\theta\) is the free neural network and the boundary factor \(g\)
is chosen so that \(gN_\theta\in V_0\). For nonhomogeneous essential data,
one may use
\begin{equation}
u_\theta(x)
=
u_b(x)+g(x)N_\theta(x),
\qquad
g=0
\quad
\text{on } \Gamma_D.
\label{eq:hard_boundary_ansatz_nonhomogeneous}
\end{equation}
The first construction yields \(u_\theta=0\) on \(\Gamma_D\), while the
second yields \(u_\theta=u_b\) there; with the stated choice of boundary
factor, both place \(u_\theta\) in \(\Acal\).

If the neural trial class is not constrained to lie in \(\Acal\) and
essential boundary conditions are imposed only through a finite-weight
penalty in the training objective, admissibility is not guaranteed. Although
the penalty may reduce essential-boundary violations, it does not in general
enforce the essential boundary condition exactly. Consequently,
\(u_\theta-u^\ast\) is not guaranteed to belong to \(V_0\), and the
residual--energy identity \eqref{eq:main_identity} does not apply directly
to the original constrained problem without an admissible correction or a
modified formulation. This contrasts with
\eqref{eq:hard_boundary_ansatz_homogeneous}--%
\eqref{eq:hard_boundary_ansatz_nonhomogeneous}, which enforce admissibility
by construction. Natural boundary conditions, such as Neumann data and
symmetric Robin terms, instead enter the variational formulation through
\(\ell(\cdot)\), or through \(a(\cdot,\cdot)\) and \(\ell(\cdot)\), rather
than defining additional affine constraints on \(u_\theta\).

\section{Residual--energy structure and conforming Riesz error assessment}
\label{sec:riesz_error_assessment}

Within the symmetric coercive setting of
Section~\ref{sec:setting}, this section develops the residual--energy
structure and its computable conforming realization for admissible neural
approximations. The analysis combines the quadratic Ritz identity,
Fr\'echet--Riesz representation, conforming Galerkin projection, and
nested-space estimation under saturation
\cite{Ciarlet1978,BrennerScott2008,BankWeiser1985,
AinsworthOden2000,Verfurth2013}.
The resulting lower monitor \(\eta_m\) and, under saturation, the
hierarchical bracket \eqref{eq:lower_upper_bracket} provide the analytical
basis for the checkpoint-selection and oracle guarantees developed in
Section~\ref{sec:checkpoint_selection}.

\subsection{Exact residual--energy identities}
\label{subsec:residual_energy_structure}

We first record the exact identities linking the Ritz functional, the
energy-dual residual, and the energy error. For admissible
\(u\in\mathcal A\), these identities identify
\(\|R(u)\|_{V_{0,a}'}\) as the residual quantity exactly calibrated with
\(\|u-u^\ast\|_a\).

\begin{lemma}[Energy-gap identity]
\label{lem:energy_gap}
For every \(u\in\Acal\),
\[
J(u)-J(u^\ast)
=
\tfrac12\|u-u^\ast\|_a^2.
\]
\end{lemma}

\begin{proof}
Set \(e:=u-u^\ast\in V_0\). Using \(u=u^\ast+e\), symmetry of \(a\),
and \(\ell(e)=a(u^\ast,e)\), we obtain
\[
\begin{aligned}
J(u)-J(u^\ast)
&=
a(u^\ast,e)+\tfrac12 a(e,e)-\ell(e) \\
&=
\tfrac12 a(e,e)
=
\tfrac12\|e\|_a^2.
\end{aligned}
\]
\end{proof}

\begin{proposition}[Exact residual--energy identity]
\label{prop:residual_identity}
For every \(u\in\Acal\),
\begin{equation}
\|R(u)\|_{V_{0,a}'}
=
\|u-u^\ast\|_a.
\label{eq:residual_error_identity}
\end{equation}
Consequently,
\begin{equation}
J(u)-J(u^\ast)
=
\tfrac12\|R(u)\|_{V_{0,a}'}^2
=
\tfrac12\|u-u^\ast\|_a^2.
\label{eq:main_identity}
\end{equation}
\end{proposition}

\begin{proof}
Set \(e:=u-u^\ast\in V_0\). For every \(v\in V_0\),
\[
\langle R(u),v\rangle
=
\ell(v)-a(u,v)
=
-a(e,v).
\]
Hence, by Cauchy--Schwarz,
\[
\|R(u)\|_{V_{0,a}'}
\leq
\|e\|_a.
\]
If \(e\neq0\), choosing \(v=-e\) in the definition of the dual norm gives
the reverse inequality. The case \(e=0\) is immediate, and
\eqref{eq:main_identity} follows from Lemma~\ref{lem:energy_gap}.
\end{proof}

Thus, for admissible approximations, the energy-dual residual norm coincides
exactly with the energy error. Minimizing
\(\|R(u)\|_{V_{0,a}'}\) over admissible candidates is therefore equivalent
to minimizing the unknown energy error, independently of how the candidates
were generated. This identifies the error-calibrated quantity targeted by the
conforming Riesz reconstruction.

\begin{remark}[Ritz ordering and error quantification]
\label{rem:ritz_functional_role}
Identity~\eqref{eq:main_identity} also shows that exact Ritz values preserve
the ordering of the energy errors. They do not, however, quantify the energy
error without the unknown reference level \(J(u^\ast)\). The energy-dual
residual norm removes this reference dependence: it equals the energy error
and admits the computable conforming realization introduced in
Subsection~\ref{subsec:conforming_riesz_hierarchy}.
\end{remark}

\begin{remark}[Strong residuals as indirect proxies]
\label{rem:strong_residual}
Suppose that the residual of an admissible approximation \(u\) admits an
\(L^2\)-representation
\[
\langle R(u),v\rangle
=
-\int_\Omega r_s(u)\cdot v\,dx,
\qquad
\forall v\in V_0,
\]
and that, for some \(C_{\rm emb}>0\),
\[
\|v\|_{L^2(\Omega)}
\leq
C_{\rm emb}\|v\|_a,
\qquad
\forall v\in V_0.
\]
Then
\[
\|R(u)\|_{V_{0,a}'}
\leq
C_{\rm emb}\|r_s(u)\|_{L^2(\Omega)}.
\]
In standard elliptic settings, such a bound follows from the relevant
Poincar\'e--Friedrichs, Korn, or norm-equivalence estimates; see, e.g.,
\cite{Ciarlet1978,BrennerScott2008}. Thus, strong-form \(L^2\) residuals may
control the energy-dual residual norm through a problem-dependent constant,
but they remain indirect proxies rather than the exactly calibrated residual
quantity.
\end{remark}

\subsection{Variational consequences for admissible trial sets}
\label{subsec:variational_consequences}

We first record a direct consequence of the exact Ritz--energy identity for
admissible trial sets.

\begin{proposition}[Best energy approximation over arbitrary admissible trial sets]
\label{prop:best_approximation}
Let \(\Sigma\subset\Acal\) be any nonempty admissible trial set, not necessarily linear or convex, and assume that
\[
u_\Sigma\in\argmin_{w\in\Sigma}J(w).
\]
Then
\[
\|u_\Sigma-u^\ast\|_a
=
\inf_{w\in\Sigma}\|w-u^\ast\|_a.
\]
\end{proposition}

\begin{proof}
Since \(\Sigma\subset\Acal\), every \(w\in\Sigma\) satisfies
\(w-u^\ast\in V_0\). By Lemma~\ref{lem:energy_gap},
\[
J(w)-J(u^\ast)
=
\tfrac12\|w-u^\ast\|_a^2,
\qquad
w\in\Sigma.
\]
Since \(J(u^\ast)\) is independent of \(w\), minimizing \(J\) over
\(\Sigma\) is equivalent to minimizing \(\|w-u^\ast\|_a\) over
\(\Sigma\).
\end{proof}

Proposition~\ref{prop:best_approximation} identifies the exact continuous
optimization geometry underlying admissible neural Ritz approximation: any
exact minimizer of the Ritz functional over an admissible trial set is an
energy-best element of that set, namely, a closest element to the exact
solution in the energy norm.

This exact statement concerns the continuous functional \(J\). If \(J\) is
replaced by an empirical or quadrature-based approximation \(\widetilde J\),
the exact ordering is not automatically preserved. Indeed, for admissible
candidates \(u_i,u_j\),
\[
J(u_i)<J(u_j)
\quad\Longleftrightarrow\quad
\|u_i-u^\ast\|_a<\|u_j-u^\ast\|_a.
\]
If
\[
\bigl|\widetilde J(u_k)-J(u_k)\bigr|
\leq\varepsilon_k,
\qquad k\in\{i,j\},
\]
then this ordering is guaranteed to be preserved whenever
\[
J(u_j)-J(u_i)>\varepsilon_i+\varepsilon_j.
\]
When this separation condition is not available, the approximation bounds
alone do not guarantee preservation of the energy-error ordering.
Consequently, a certified connection between the practical training problem
and exact Ritz minimization requires quantitative control of both the
functional-approximation error and the optimization error.

This motivates a post-training assessment targeted directly at the energy
error. By Proposition~\ref{prop:residual_identity}, the energy-dual residual
norm provides an exactly calibrated target without access to \(u^\ast\);
the next subsection constructs its computable conforming approximation.

\begin{remark}[Relation to Ritz--Galerkin best approximation]
For \(\Sigma=u_b+W_h\), with \(W_h\subset V_0\) a conforming linear
subspace, Proposition~\ref{prop:best_approximation} reduces to the classical
Ritz--Galerkin best-approximation property
\cite{Ciarlet1978,BrennerScott2008}.
Here the same conclusion holds for any nonempty admissible
\(\Sigma\subset\mathcal A\), without linearity or convexity assumptions.
\end{remark}

\begin{proposition}[Near-minimizer control]
\label{prop:near_minimizer}
Let \(u\in\Acal\) and \(\varepsilon\geq0\). If
\[
J(u)-J(u^\ast)\leq\varepsilon,
\]
then
\[
\|u-u^\ast\|_a\leq\sqrt{2\varepsilon},
\qquad
\|R(u)\|_{V_{0,a}'}\leq\sqrt{2\varepsilon}.
\]
\end{proposition}

\begin{proof}
The result follows directly from Lemma~\ref{lem:energy_gap} and
Proposition~\ref{prop:residual_identity}.
\end{proof}

\begin{corollary}[Residual convergence and energy-error control]
\label{cor:residual_convergence}
Let \((u_n)_n\subset\mathcal A\). Then:
\begin{enumerate}[label=(\roman*),leftmargin=*]
\item if
\(\|R(u_n)\|_{V_{0,a}'}\to0\), then
\(u_n\to u^\ast\) in the energy norm;
\item if \(u\in\mathcal A\) and
\(\|R(u)\|_{V_{0,a}'}\leq\varepsilon\), then
\(\|u-u^\ast\|_a\leq\varepsilon\).
\end{enumerate}
\end{corollary}

\begin{proof}
Both statements follow directly from
Proposition~\ref{prop:residual_identity}.
\end{proof}

\subsection{Conforming Riesz reconstruction and hierarchical error bounds}
\label{subsec:conforming_riesz_hierarchy}

We now turn the exact residual--energy relation into a computable
post-training assessment. By
Corollary~\ref{cor:residual_convergence}, reducing
\(\|R(u_\theta)\|_{V_{0,a}'}\) drives the energy error to zero, while any
bound on this residual norm transfers directly to
\(\|u_\theta-u^\ast\|_a\). We therefore approximate the Riesz representative
of \(R(u_\theta)\) in conforming auxiliary subspaces of \(V_0\).
A single auxiliary space yields a lower monitor of the resolved energy-error
component, while a nested pair quantifies the component revealed by
enrichment and, under saturation, provides a conditional upper estimate.

\subsubsection{Continuous Riesz representative}

For an admissible neural Ritz approximation \(u_\theta\in\mathcal A\), the
residual \(R(u_\theta)\) belongs to \(V_0'\). By the Riesz representation
theorem in the Hilbert space \((V_0,\|\cdot\|_a)\), there exists a unique
\(z\in V_0\) such that
\begin{equation}
a(z,v)
=
\langle R(u_\theta),v\rangle,
\qquad
\forall v\in V_0.
\label{eq:continuous_riesz}
\end{equation}

\begin{proposition}[Continuous Riesz representative as the signed error]
\label{prop:continuous_riesz_error}
Let \(u_\theta\in\mathcal A\), and let \(z\in V_0\) solve
\eqref{eq:continuous_riesz}. Then
\[
z=u^\ast-u_\theta,
\]
and consequently
\begin{equation}
\|z\|_a
=
\|u_\theta-u^\ast\|_a
=
\|R(u_\theta)\|_{V_{0,a}'}.
\label{eq:continuous_riesz_error_norm}
\end{equation}
\end{proposition}

\begin{proof}
Since \(u_\theta,u^\ast\in\mathcal A\),
\(u^\ast-u_\theta\in V_0\). For every \(v\in V_0\),
\[
a(z,v)
=
\langle R(u_\theta),v\rangle
=
\ell(v)-a(u_\theta,v)
=
a(u^\ast-u_\theta,v).
\]
Uniqueness of the Riesz representative gives
\(z=u^\ast-u_\theta\). The norm identities follow from
Proposition~\ref{prop:residual_identity}.
\end{proof}

Thus, the exact signed error is the continuous Riesz representative of the
residual. The computational problem is therefore to approximate this
representative in a conforming auxiliary space.

\subsubsection{Single-space conforming realization}

Let \(V_m\subset V_0\) be a finite-dimensional conforming auxiliary
subspace. The discrete Riesz representative \(z_m\in V_m\) is defined by
\begin{equation}
a(z_m,v_m)
=
\langle R(u_\theta),v_m\rangle,
\qquad
\forall v_m\in V_m.
\label{eq:discrete_riesz}
\end{equation}
We define
\begin{equation}
\eta_m(u_\theta):=\|z_m\|_a.
\label{eq:eta_definition}
\end{equation}

Because the continuous Riesz representative satisfies
\(z=u^\ast-u_\theta\), an exact auxiliary solve would recover the signed
primal error. In the present selection framework, however, \(z_m\) enters
only through the scalar monitor \(\eta_m(u_\theta)=\|z_m\|_a\).
Accordingly, archive selection does not require exact reconstruction of
\(z\); it requires only sufficient resolution to preserve the
oracle--non-oracle ordering. This distinction is quantified in
Corollary~\ref{cor:oracle_recovery}.

The distinction from direct Ritz--Galerkin approximation
concerns both the objective and the required accuracy.
A direct method constructs a primal approximation
$u_h\in u_b+W_h$ satisfying
\[
a(u_h,v_h)=\ell(v_h),
\qquad \forall v_h\in W_h,
\]
whose accuracy may be assessed against a prescribed
energy-error tolerance $\|u_h-u^\ast\|_a\leq\tau$.
In the present framework, the neural checkpoints are
already computed and remain fixed, while auxiliary
residual problems provide the scalar monitors used
for their selection. Indeed, \eqref{eq:discrete_riesz}
gives
\[
a(u_\theta+z_m,v_m)=\ell(v_m),
\qquad \forall v_m\in V_m,
\]
so that $u_\theta+z_m$ is a Ritz--Galerkin
approximation in the translated affine space
$u_\theta+V_m$. The selected output, however, is
an original neural checkpoint, not this corrected
field. 
Accordingly, the relevant question is whether the
auxiliary monitors distinguish the energy-optimal
checkpoints from the remaining archive, rather than
whether a corrected field meets a prescribed primal
error tolerance. Corollary~\ref{cor:oracle_recovery}
provides a sufficient condition for exact
logged-oracle selection in terms of the uniform
monitor deficit and the oracle--non-oracle gap.

Since \(V_m\subset V_0\), problem~\eqref{eq:discrete_riesz} is the
conforming Galerkin approximation of the continuous Riesz problem
\eqref{eq:continuous_riesz}. Galerkin orthogonality identifies \(z_m\)
as the \(a\)-orthogonal projection of \(z\) onto \(V_m\)
\cite{Ciarlet1978,BrennerScott2008}. In the present residual--Riesz
setting, this projection yields the computable lower-monitor properties
stated below.

\begin{proposition}[Conforming projection and lower monitor]
\label{prop:projection_decomposition}
Let \(z\in V_0\) and \(z_m\in V_m\) solve
\eqref{eq:continuous_riesz} and \eqref{eq:discrete_riesz},
respectively. Then
\begin{equation}
a(z-z_m,v_m)=0,
\qquad
\forall v_m\in V_m,
\label{eq:riesz_galerkin_orthogonality}
\end{equation}
so that \(z_m\) is the best approximation of \(z\) in \(V_m\):
\[
\|z-z_m\|_a
=
\inf_{v_m\in V_m}\|z-v_m\|_a.
\]
Moreover,
\begin{equation}
\|u_\theta-u^\ast\|_a^2
=
\eta_m(u_\theta)^2
+
\|z-z_m\|_a^2,
\label{eq:pythagorean}
\end{equation}
and therefore
\begin{equation}
\eta_m(u_\theta)
\leq
\|u_\theta-u^\ast\|_a
=
\|R(u_\theta)\|_{V_{0,a}'}.
\label{eq:lower_bound_monitor}
\end{equation}
\end{proposition}

\begin{proof}
Subtracting \eqref{eq:discrete_riesz} from
\eqref{eq:continuous_riesz} gives
\eqref{eq:riesz_galerkin_orthogonality}. Hence \(z_m\) is the
\(a\)-orthogonal projection of \(z\) onto \(V_m\), so
\[
\|z\|_a^2
=
\|z_m\|_a^2+\|z-z_m\|_a^2.
\]
Using Proposition~\ref{prop:continuous_riesz_error} and
\(\eta_m(u_\theta)=\|z_m\|_a\) yields
\eqref{eq:pythagorean} and \eqref{eq:lower_bound_monitor}.
\end{proof}

Thus, solving the finite-dimensional auxiliary problem
\eqref{eq:discrete_riesz} yields \(\eta_m(u_\theta)\) as a computable
conforming lower monitor for the energy error, while \(\|z-z_m\|_a\) is the
component unresolved by \(V_m\). This decomposition motivates nested
auxiliary spaces, which quantify the additional component revealed by
enrichment.

\begin{corollary}[Monotone recovery under conforming refinement]
\label{cor:conforming_refinement}
Let \((V_m)_{m\geq1}\) be a nested sequence of finite-dimensional conforming
subspaces such that
\[
V_m\subset V_{m+1}\subset V_0,
\qquad
\overline{\bigcup_{m\geq1}V_m}^{\|\cdot\|_a}=V_0.
\]
For a fixed \(u_\theta\in\mathcal A\), let \(z_m\in V_m\) solve
\eqref{eq:discrete_riesz}. Then \(\eta_m(u_\theta)=\|z_m\|_a\) is
nondecreasing and
\begin{equation}
\eta_m(u_\theta)
\uparrow
\|u_\theta-u^\ast\|_a
=
\|R(u_\theta)\|_{V_{0,a}'}.
\label{eq:monotone_recovery}
\end{equation}
\end{corollary}

\begin{proof}
By Proposition~\ref{prop:projection_decomposition},
\[
\eta_m(u_\theta)^2
=
\|z\|_a^2-\|z-z_m\|_a^2.
\]
Nestedness makes \(\|z-z_m\|_a\) nonincreasing, while density and the
best-approximation property imply \(\|z-z_m\|_a\to0\).
Hence \(\eta_m(u_\theta)\) increases to \(\|z\|_a\), yielding
\eqref{eq:monotone_recovery}.
\end{proof}

Thus, \(\eta_m(u_\theta)\) recovers the exact energy error monotonically
from below under conforming refinement. Over a finite checkpoint archive,
this property provides the basis for the reference-free checkpoint-selection
rule and oracle guarantees developed in
Section~\ref{sec:checkpoint_selection}.

A finite enrichment \(V_m\subset V_M\) is considered next to quantify the
additional error component resolved beyond \(V_m\).

\subsubsection{Nested-space enrichment and hierarchical gap}

Let \(V_m\subset V_M\subset V_0\) be two finite-dimensional conforming auxiliary subspaces.
In addition to the reconstruction \(z_m\in V_m\), define the enriched Riesz reconstruction \(z_M\in V_M\) by
\begin{equation}
a(z_M,v_M)
=
\langle R(u_\theta),v_M\rangle,
\qquad
\forall v_M\in V_M.
\label{eq:discrete_riesz_large}
\end{equation}
We set
\[
\eta_M(u_\theta):=\|z_M\|_a
\]
and define the hierarchical gap
\begin{equation}
\delta_{m,M}(u_\theta)
:=
\left(
\eta_M(u_\theta)^2-\eta_m(u_\theta)^2
\right)^{1/2}.
\label{eq:hierarchical_gap}
\end{equation}

\begin{proposition}[Nested-space orthogonal decomposition]
\label{prop:two_space_decomposition}
Let \(u_\theta\in\mathcal A\), let \(z\in V_0\) solve
\eqref{eq:continuous_riesz}, and let \(z_m\in V_m\) and \(z_M\in V_M\) solve
\eqref{eq:discrete_riesz} and \eqref{eq:discrete_riesz_large}, respectively.
If \(V_m\subset V_M\subset V_0\), then
\[
\eta_m(u_\theta)^2
\leq
\eta_M(u_\theta)^2
\leq
\|u_\theta-u^\ast\|_a^2.
\]
Moreover,
\begin{equation}
\delta_{m,M}(u_\theta)^2
=
\eta_M(u_\theta)^2-\eta_m(u_\theta)^2
=
\|z_M-z_m\|_a^2,
\label{eq:two_space_gap_identity}
\end{equation}
and the exact energy error admits the orthogonal decomposition
\begin{equation}
\|u_\theta-u^\ast\|_a^2
=
\eta_m(u_\theta)^2
+
\delta_{m,M}(u_\theta)^2
+
\|z-z_M\|_a^2.
\label{eq:three_term_decomposition}
\end{equation}
\end{proposition}

\begin{proof}
Since \(V_m\subset V_M\), subtracting the two discrete Riesz problems gives
\[
a(z_M-z_m,v_m)=0,
\qquad
\forall v_m\in V_m.
\]
Hence
\[
\|z_M\|_a^2
=
\|z_m\|_a^2+\|z_M-z_m\|_a^2,
\]
which yields \eqref{eq:two_space_gap_identity} and
\(\eta_m(u_\theta)\leq\eta_M(u_\theta)\). Moreover, Galerkin
orthogonality at level \(M\) gives
\[
\|z\|_a^2
=
\|z_M\|_a^2+\|z-z_M\|_a^2.
\]
Combining the two identities and using \(z=u^\ast-u_\theta\) yields
\eqref{eq:three_term_decomposition} and the remaining inequality.
\end{proof}

\begin{remark}[Hierarchical interpretation]
The decomposition in
Proposition~\ref{prop:two_space_decomposition} is the nested-projection
geometry underlying hierarchical a posteriori estimation
\cite{BankWeiser1985,AinsworthOden2000,Verfurth1996}.
Here the projected quantity is the signed energy error:
\(\eta_m(u_\theta)\) measures the component resolved in \(V_m\),
\(\delta_{m,M}(u_\theta)\) the additional component revealed by enrichment
to \(V_M\), and \(\|z-z_M\|_a\) the remaining unresolved component.

Because the auxiliary hierarchy is independent of the network
parametrization and training procedure, the same construction can be applied
post-training to a finite checkpoint archive. Under saturation,
\(\eta_m(u_\theta)\) and \(\delta_{m,M}(u_\theta)\) then provide the
lower--upper information used for the finite-level selection guarantees of
Section~\ref{sec:checkpoint_selection}.
\end{remark}

To control the unresolved component \(\|z-z_M\|_a\) through the computable
gap \(\delta_{m,M}(u_\theta)\), nestedness alone is insufficient. We therefore
introduce the standard saturation condition from hierarchical a posteriori
estimation
\cite{BankWeiser1985,AinsworthOden2000,Verfurth1996}.
It enters only the conditional upper estimate; the lower monitor and its
monotone recovery remain unconditional.

\begin{assumption}[Saturation]
\label{ass:saturation}
Let \(u_\theta\in\mathcal A\) be fixed, with continuous Riesz representative
\(z=u^\ast-u_\theta\). A nested pair
\(V_m\subset V_M\subset V_0\) satisfies saturation for \(u_\theta\), with
factor \(q\in(0,1)\), if
\begin{equation}
\|z-z_M\|_a
\leq
q\,\|z-z_m\|_a.
\label{eq:saturation}
\end{equation}
\end{assumption}

\begin{proposition}[Conditional upper estimate under saturation]
\label{prop:saturated_upper_estimate}
Let \(u_\theta\in\mathcal A\), and suppose that the nested pair
\(V_m\subset V_M\subset V_0\) satisfies
Assumption~\ref{ass:saturation} for \(u_\theta\), with factor
\(q\in(0,1)\). Then
\begin{equation}
\|u_\theta-u^\ast\|_a^2
\leq
\eta_m(u_\theta)^2
+
\frac{\delta_{m,M}(u_\theta)^2}{1-q^2}.
\label{eq:saturated_upper_bound_squared}
\end{equation}
Accordingly, defining
\begin{equation}
\mathcal U_{m,M,q}(u_\theta)
:=
\left(
\eta_m(u_\theta)^2
+
\frac{\delta_{m,M}(u_\theta)^2}{1-q^2}
\right)^{1/2},
\label{eq:saturated_upper_estimator_definition}
\end{equation}
one has
\begin{equation}
\|u_\theta-u^\ast\|_a
\leq
\mathcal U_{m,M,q}(u_\theta).
\label{eq:saturated_upper_estimator}
\end{equation}
\end{proposition}

\begin{proof}
By Proposition~\ref{prop:two_space_decomposition},
\[
\delta_{m,M}(u_\theta)^2
=
\|z-z_m\|_a^2-\|z-z_M\|_a^2.
\]
Assumption~\ref{ass:saturation} therefore gives
\[
\delta_{m,M}(u_\theta)^2
\geq
(1-q^2)\|z-z_m\|_a^2,
\]
and hence
\[
\|z-z_m\|_a^2
\leq
\frac{\delta_{m,M}(u_\theta)^2}{1-q^2}.
\]
Combining this with
\[
\|u_\theta-u^\ast\|_a^2
=
\eta_m(u_\theta)^2+\|z-z_m\|_a^2
\]
yields \eqref{eq:saturated_upper_bound_squared} and
\eqref{eq:saturated_upper_estimator}.
\end{proof}

Combining Proposition~\ref{prop:projection_decomposition} with
Proposition~\ref{prop:saturated_upper_estimate} yields the
lower--upper bracket
\begin{equation}
\eta_m(u_\theta)
\leq
\|u_\theta-u^\ast\|_a
\leq
\mathcal U_{m,M,q}(u_\theta).
\label{eq:lower_upper_bracket}
\end{equation}
The lower bound is unconditional, whereas the upper bound at a given
checkpoint requires Assumption~\ref{ass:saturation} for that checkpoint and
the prescribed factor \(q\). In the finite-archive guarantees developed in
Section~\ref{sec:checkpoint_selection}, saturation is required only where the
corresponding upper estimate is invoked.

\begin{remark}[Dependence on enrichment and eventual saturation]
\label{rem:eventual_saturation}
Nestedness alone does not guarantee saturation for an arbitrary finite pair.
For fixed \(V_m\), however, let \((V_M)_{M\geq m}\) be a nested enrichment
satisfying
\[
\overline{\bigcup_{M\geq m}V_M}^{\,\|\cdot\|_a}=V_0.
\]
For each fixed checkpoint \(k\),
\[
\|z_k-z_{M,k}\|_a\longrightarrow0,
\]
while \(\|z_k-z_{m,k}\|_a\) is fixed. Hence, for every prescribed
\(q\in(0,1)\), saturation holds for all sufficiently large \(M\); if the
coarse projection error vanishes, it holds trivially for every \(M\geq m\).
For a finite checkpoint archive, the corresponding thresholds can be
maximized to obtain a common enrichment level. This argument does not provide
a computable threshold, so the upper estimate for a prescribed finite pair
remains conditional.
\end{remark}

\paragraph{Operational interpretation of the hierarchical bracket}

Under the saturation condition, the bracket implies
\[
0
\leq
\|u_\theta-u^\ast\|_a-\eta_m(u_\theta)
\leq
\mathcal U_{m,M,q}(u_\theta)-\eta_m(u_\theta).
\]
Thus, its width bounds the possible underestimation of the energy error by
the lower monitor. A wide bracket indicates that further auxiliary refinement
may be required. Applied over a finite checkpoint archive, these intervals
provide the finite-level information used in
Section~\ref{sec:checkpoint_selection}.

\begin{algorithm*}[t]
\caption{Evaluation of the conforming Riesz monitor and hierarchical estimator}
\label{alg:riesz_monitor}
\begin{algorithmic}[1]
\Require Admissible approximation \(u_\theta\in\mathcal A\);
preassembled auxiliary matrices \(G_m\) and \(G_M\);
data for assembling the residual vectors on \(V_m\) and \(V_M\);
optional prescribed factor \(q\in(0,1)\).
\Ensure Lower monitor \(\eta_m\), enriched monitor \(\eta_M\),
hierarchical gap \(\delta_{m,M}\), and, when \(q\) is prescribed,
conditional upper quantity \(\mathcal U_{m,M,q}\).

\For{\(h\in\{m,M\}\)}
    \State Assemble the checkpoint-dependent residual vector \(r_h\).
    \State Solve \(G_hc^{(h)}=r_h\).
    \State Compute
    \[
    \eta_h(u_\theta)
    =
    \bigl(r_h^\top c^{(h)}\bigr)^{1/2}.
    \]
\EndFor

\State Compute
\[
\delta_{m,M}(u_\theta)
=
\bigl(\eta_M(u_\theta)^2-\eta_m(u_\theta)^2\bigr)^{1/2}.
\]

\If{\(q\in(0,1)\) is prescribed}
    \State Compute
    \[
    \mathcal U_{m,M,q}(u_\theta)
    =
    \left(
    \eta_m(u_\theta)^2
    +
    \frac{\delta_{m,M}(u_\theta)^2}{1-q^2}
    \right)^{1/2}.
    \]
\EndIf

\State \Return
\(\eta_m(u_\theta)\), \(\eta_M(u_\theta)\),
\(\delta_{m,M}(u_\theta)\), and, when computed,
\(\mathcal U_{m,M,q}(u_\theta)\).
\end{algorithmic}
\end{algorithm*}

\subsubsection{Matrix realization}

For \(h\in\{m,M\}\), let
\(\{\phi_i^{(h)}\}_{i=1}^{N_h}\) be a basis of \(V_h\), and write
\[
z_h
=
\sum_{j=1}^{N_h} c_j^{(h)}\phi_j^{(h)}.
\]
Define the auxiliary stiffness matrix
\(G_h\in\mathbb R^{N_h\times N_h}\) and the checkpoint-dependent residual
vector \(r_h\in\mathbb R^{N_h}\) by
\[
(G_h)_{ij}
=
a\bigl(\phi_j^{(h)},\phi_i^{(h)}\bigr),
\qquad
(r_h)_i
=
\ell\bigl(\phi_i^{(h)}\bigr)
-
a\bigl(u_\theta,\phi_i^{(h)}\bigr).
\]
The discrete Riesz problem is then equivalent to
\begin{equation}
G_hc^{(h)}=r_h.
\label{eq:matrix_system_h}
\end{equation}
By symmetry and coercivity of \(a\), \(G_h\) is symmetric positive
definite. Once \eqref{eq:matrix_system_h} has been solved, the monitor is
obtained from
\begin{equation}
\eta_h(u_\theta)
=
\|z_h\|_a
=
\left((c^{(h)})^\top G_hc^{(h)}\right)^{1/2}
=
\left(r_h^\top c^{(h)}\right)^{1/2}.
\label{eq:matrix_monitor}
\end{equation}

The matrices \(G_m\) and \(G_M\) depend only on the bilinear form and the
auxiliary spaces and may therefore be preassembled and factorized once. For
each checkpoint, only the checkpoint-dependent residual vectors are assembled
and the corresponding linear systems solved. Applying the construction at
\(h=m\) and \(h=M\) yields \(\eta_m(u_\theta)\) and
\(\eta_M(u_\theta)\), from which \(\delta_{m,M}(u_\theta)\) and, for a
prescribed \(q\in(0,1)\), \(\mathcal U_{m,M,q}(u_\theta)\) follow as
summarized in Algorithm~\ref{alg:riesz_monitor}.

The identities above refer to the exact auxiliary variational forms; their
numerical realization under quadrature is assessed separately in
Section~\ref{sec:numerical} and~\ref{app:discrete_quadrature_audit}.

\section{Reference-free checkpoint selection and logged-oracle guarantees}
\label{sec:checkpoint_selection}

This section develops the reference-free selection rule by minimizing the
conforming Riesz monitor over a finite archive. We identify the
finite-resolution ordering obstruction, show that conforming refinement
removes it uniformly and yields eventual logged-oracle selection, and derive,
under saturation, a computable near-oracle bound and a unique-oracle
certificate.

\subsection{Monitor minimization and finite-resolution ordering obstruction}
\label{subsec:monitor_minimization_obstruction}

Let \(\mathcal K_{\mathrm{log}}\) denote the finite set of recorded
checkpoints and, for brevity, write \(u_k:=u_{\theta_k}\). Define
\[
E_k:=\|u^\ast-u_k\|_a,
\qquad
k\in\mathcal K_{\mathrm{log}}.
\]
By Proposition~\ref{prop:projection_decomposition} and
Corollary~\ref{cor:conforming_refinement}, for every fixed
\(k\in\mathcal K_{\mathrm{log}}\),
\begin{equation}
\eta_{m,k}\leq E_k,
\qquad
\eta_{m,k}\longrightarrow E_k
\quad\text{as }m\to\infty.
\label{eq:checkpointwise_monitor_recovery}
\end{equation}

These checkpointwise properties motivate the computable selection rule
\[
k_m^\ast
\in
\operatorname*{arg\,min}_{k\in\mathcal K_{\mathrm{log}}}
\eta_{m,k}.
\]
The corresponding consistency question is whether, for all sufficiently fine
auxiliary resolutions,
\[
\operatorname*{arg\,min}_{k\in\mathcal K_{\mathrm{log}}}
\eta_{m,k}
\subseteq
\operatorname*{arg\,min}_{k\in\mathcal K_{\mathrm{log}}}
E_k.
\]

At a fixed auxiliary level, checkpointwise recovery alone does not guarantee
this ordering. Let \(i\) be an energy-oracle checkpoint and \(j\) a
non-oracle checkpoint, so that \(E_i<E_j\). Although both monitors are valid
lower bounds, one may still have
\[
\eta_{m,j}<\eta_{m,i},
\]
because the unresolved error component may differ across checkpoints.

To characterize this obstruction, define
\[
z_k:=u^\ast-u_k,
\qquad
\rho_{m,k}:=\|(I-P_m)z_k\|_a,
\]
where \(P_m\) denotes the \(a\)-orthogonal projection onto \(V_m\).
Then
\[
E_k^2
=
\eta_{m,k}^2+\rho_{m,k}^2.
\]

\begin{proposition}[Finite-resolution ordering obstruction]
\label{prop:finite-resolution-rank-reversal}
For any two checkpoints \(i,j\in\mathcal K_{\mathrm{log}}\),
\[
\eta_{m,j}^2-\eta_{m,i}^2
=
\bigl(E_j^2-E_i^2\bigr)
-
\bigl(\rho_{m,j}^2-\rho_{m,i}^2\bigr).
\]
Consequently, if \(E_i<E_j\), then the monitor ordering is reversed,
\[
\eta_{m,i}>\eta_{m,j},
\]
if and only if
\[
\rho_{m,j}^2-\rho_{m,i}^2
>
E_j^2-E_i^2.
\]

Moreover, if
\[
\{0\}\subsetneq V_m\subsetneq V_0,
\]
then such a reversal is compatible with the present assumptions: there
exist admissible approximations \(u_i,u_j\in\mathcal A\) such that
\[
E_i<E_j
\qquad\text{and}\qquad
0<\eta_{m,j}<\eta_{m,i}.
\]
\end{proposition}

\begin{proof}
For each checkpoint \(k\), the \(a\)-orthogonal decomposition
\[
z_k=P_mz_k+(I-P_m)z_k
\]
gives
\[
\eta_{m,k}^2
=
E_k^2-\rho_{m,k}^2.
\]
Subtracting the identities for \(i\) and \(j\) yields the first claim and,
when \(E_i<E_j\), the stated characterization of rank reversal.

To show that such a reversal can occur, choose \(a\)-unit vectors
\[
x\in V_m,
\qquad
y\in V_m^{\perp_a}\cap V_0,
\]
and let \(0<\gamma<\alpha<\beta\). Define
\[
z_i:=\alpha x,
\qquad
z_j:=\gamma x+\beta y,
\qquad
u_i:=u^\ast-z_i,
\qquad
u_j:=u^\ast-z_j.
\]
Then \(u_i,u_j\in\mathcal A\), and
\[
E_i=\alpha
<
\sqrt{\gamma^2+\beta^2}
=
E_j.
\]
Since
\[
P_mz_i=\alpha x,
\qquad
P_mz_j=\gamma x,
\]
one has
\[
\eta_{m,i}=\alpha>\gamma=\eta_{m,j}>0.
\]
\end{proof}

Thus, the finite-resolution obstruction is caused by checkpoint-dependent
projection defects, not by failure of the lower-bound property itself.
At fixed \(V_m\), the unresolved component \(\rho_{m,k}\) depends on the
direction of the error \(z_k\) relative to the auxiliary space, not only on
its norm. Different checkpoints may therefore have different resolved and unresolved
error components, allowing the unresolved-component difference to overcome
the exact energy gap and reverse the ordering.

\subsection{Logged-oracle consistency under conforming refinement}
\label{subsec:logged_oracle_consistency}

We now show that the finite-resolution obstruction identified in
Proposition~\ref{prop:finite-resolution-rank-reversal} cannot persist under a
common nested conforming refinement of a finite checkpoint archive. The key
point is that checkpointwise recovery becomes uniform over finitely many
checkpoints.

Define the logged energy-oracle level and oracle set by
\[
E_\ast
:=
\min_{k\in\mathcal K_{\mathrm{log}}}E_k,
\qquad
\mathcal K_E^\ast
:=
\operatorname*{arg\,min}_{k\in\mathcal K_{\mathrm{log}}}E_k.
\]
Since \(u^\ast\) is unknown, \(E_\ast\) and
\(\mathcal K_E^\ast\) are unavailable in practical computations.

At auxiliary level \(m\), define
\[
\mathcal K_m^\ast
:=
\operatorname*{arg\,min}_{k\in\mathcal K_{\mathrm{log}}}
\eta_{m,k}.
\]
A reference-free checkpoint may therefore be selected as any
\(k_m^\ast\in\mathcal K_m^\ast\). This rule uses only the unconditional
lower monitor and is independent of saturation, which enters only the
finite-resolution guarantees developed later.

\begin{lemma}[Uniform recovery over the logged set]
\label{lem:uniform_checkpoint_recovery}
Under the assumptions of
Corollary~\ref{cor:conforming_refinement}, define
\[
\varepsilon_m
:=
\max_{k\in\mathcal K_{\mathrm{log}}}
\bigl(E_k-\eta_{m,k}\bigr).
\]
Then
\[
\varepsilon_m
=
\max_{k\in\mathcal K_{\mathrm{log}}}
\left|E_k-\eta_{m,k}\right|
\longrightarrow0
\qquad\text{as }m\to\infty.
\]
\end{lemma}

\begin{proof}
For every fixed \(k\in\mathcal K_{\mathrm{log}}\),
Corollary~\ref{cor:conforming_refinement} gives
\[
0\leq E_k-\eta_{m,k}\longrightarrow0.
\]
Since \(\mathcal K_{\mathrm{log}}\) is finite, the maximum of these finitely
many convergent defects also tends to zero.
\end{proof}

Thus,
\[
0
\leq
E_k-\eta_{m,k}
\leq
\varepsilon_m
\qquad
\forall k\in\mathcal K_{\mathrm{log}},
\]
with \(\varepsilon_m\to0\). Combined with monitor minimality, this uniform
control yields value consistency.

\begin{proposition}[Value consistency of monitor-based selection]
\label{prop:monitor_selection_consistency}
Let \(k_m^\ast\in\mathcal K_m^\ast\) be any monitor-minimizing checkpoint.
Then
\[
0
\leq
E_{k_m^\ast}-E_\ast
\leq
\varepsilon_m.
\]
Consequently,
\[
E_{k_m^\ast}\longrightarrow E_\ast
\qquad\text{as }m\to\infty.
\]
\end{proposition}

\begin{proof}
Let \(k_E^\ast\in\mathcal K_E^\ast\). By monitor minimality, the lower-bound
property, and the definition of \(\varepsilon_m\),
\[
E_{k_m^\ast}
\leq
\eta_{m,k_m^\ast}+\varepsilon_m
\leq
\eta_{m,k_E^\ast}+\varepsilon_m
\leq
E_\ast+\varepsilon_m.
\]
Since \(E_\ast\leq E_{k_m^\ast}\), the stated bound follows.
\end{proof}

Thus, monitor minimization is asymptotically consistent in value. Exact
checkpointwise error recovery is stronger than required for oracle selection:
it suffices that the unresolved monitor error fall below the
oracle--non-oracle gap. The following corollary makes this threshold explicit.

\begin{corollary}[Finite-gap criterion and eventual logged-oracle selection]
\label{cor:oracle_recovery}
Under the assumptions of
Corollary~\ref{cor:conforming_refinement}, suppose first that
\[
\mathcal K_E^\ast\neq\mathcal K_{\mathrm{log}},
\]
and define the oracle--non-oracle gap
\[
\delta_E
:=
\min_{k\in
\mathcal K_{\mathrm{log}}\setminus\mathcal K_E^\ast}
\bigl(E_k-E_\ast\bigr)
>0.
\]
At every auxiliary level \(m\) such that
\[
\varepsilon_m<\delta_E,
\]
one has
\[
\mathcal K_m^\ast\subseteq\mathcal K_E^\ast.
\]
Consequently, in all cases there exists \(m_0\) such that
\[
\mathcal K_m^\ast
\subseteq
\mathcal K_E^\ast
\qquad
\forall m\geq m_0.
\]
If the logged energy oracle is unique,
\[
\mathcal K_E^\ast=\{k_E^\ast\},
\]
then
\[
\mathcal K_m^\ast=\{k_E^\ast\}
\qquad
\forall m\geq m_0.
\]
\end{corollary}

\begin{proof}
If \(\mathcal K_E^\ast=\mathcal K_{\mathrm{log}}\), the conclusion is
immediate. Otherwise, let \(k_m^\ast\in\mathcal K_m^\ast\). By
Proposition~\ref{prop:monitor_selection_consistency},
\[
E_{k_m^\ast}-E_\ast
\leq
\varepsilon_m.
\]
Hence, whenever \(\varepsilon_m<\delta_E\),
\[
E_{k_m^\ast}-E_\ast<\delta_E.
\]
By definition of \(\delta_E\), no checkpoint outside
\(\mathcal K_E^\ast\) can satisfy this inequality. Therefore
\[
k_m^\ast\in\mathcal K_E^\ast,
\]
and thus
\[
\mathcal K_m^\ast\subseteq\mathcal K_E^\ast.
\]
Finally, Lemma~\ref{lem:uniform_checkpoint_recovery} gives
\(\varepsilon_m\to0\), so the finite-gap condition holds for all
sufficiently large \(m\). The unique-oracle statement follows immediately.
\end{proof}

\begin{corollary}[Selection stability under numerical monitor perturbations]
\label{cor:numerical_monitor_stability}
Let \(\mathcal K_{\mathrm{log}}\) be a fixed finite archive,
and let \(\widehat\eta_{m,k}\) denote numerically evaluated
approximations of the conforming Riesz monitors
\(\eta_{m,k}\). Suppose that, for some \(\tau_m\geq0\),
\[
\max_{k\in\mathcal K_{\mathrm{log}}}
|\widehat\eta_{m,k}-\eta_{m,k}|
\leq\tau_m.
\]
Let
\[
\widehat k_m\in
\operatorname*{arg\,min}_{k\in\mathcal K_{\mathrm{log}}}
\widehat\eta_{m,k}.
\]
Then
\[
0\leq E_{\widehat k_m}-E_\ast
\leq\varepsilon_m+2\tau_m,
\]
where \(\varepsilon_m\) is the uniform projection
deficit of Lemma~\ref{lem:uniform_checkpoint_recovery}.

If \(\mathcal K_E^\ast\neq\mathcal K_{\mathrm{log}}\),
the condition
\[
\varepsilon_m+2\tau_m<\delta_E
\]
guarantees
\[
\widehat k_m\in\mathcal K_E^\ast.
\]
Consequently, if \(\tau_m\to0\) along the conforming
refinement hierarchy, every minimizer of the
numerically evaluated monitors belongs to
\(\mathcal K_E^\ast\) for all sufficiently large \(m\).
\end{corollary}

\begin{proof}
Choose \(s\in\mathcal K_E^\ast\).
By the uniform projection bound, the assumed
numerical perturbation bound, and the minimality
of \(\widehat k_m\),
\[
\begin{aligned}
E_{\widehat k_m}
&\leq\eta_{m,\widehat k_m}+\varepsilon_m\\
&\leq\widehat\eta_{m,\widehat k_m}
       +\tau_m+\varepsilon_m\\
&\leq\widehat\eta_{m,s}
       +\tau_m+\varepsilon_m\\
&\leq\eta_{m,s}
       +2\tau_m+\varepsilon_m\\
&\leq E_\ast+2\tau_m+\varepsilon_m.
\end{aligned}
\]
The lower inequality follows from the definition
of \(E_\ast\).

If \(\varepsilon_m+2\tau_m<\delta_E\),
the selected checkpoint cannot be a non-oracle,
by the definition of the oracle--non-oracle gap.
Finally, \(\varepsilon_m\to0\) under conforming
refinement, so \(\tau_m\to0\) gives eventual
logged-oracle selection.
\end{proof}

\begin{remark}[Scope of numerical perturbation control]
The preceding result is conditional on a uniform
bound for the numerical evaluation errors.
It does not require the numerically evaluated
monitors to remain lower bounds or monotone under
refinement. In particular, the exact conforming
properties of \(\eta_{m,k}\) do not automatically
transfer to \(\widehat\eta_{m,k}\).

Obtaining a verified value of \(\tau_m\) requires
control of the errors introduced by residual
assembly, numerical quadrature, and the auxiliary
linear solves. Without such control, the result
provides a perturbation guarantee rather than a
fully computable numerical certificate.
\end{remark}

\begin{remark}[Scope of logged-oracle recovery]
\label{rem:scope_logged_oracle}
Proposition~\ref{prop:monitor_selection_consistency} and
Corollary~\ref{cor:oracle_recovery} concern the oracle over the prescribed
archive \(\mathcal K_{\mathrm{log}}\). If this archive contains a minimizer
of the energy error over the computed training trajectory, then its logged
oracle is also trajectory-wide. Whether such a minimizer is retained is the
distinct archive-coverage question addressed in
Section~\ref{sec:trajectory_wide_control}.
\end{remark}

\begin{remark}[Geometric removal of the ordering obstruction]
\label{rem:geometric_oracle_recovery}
If \(\mathcal K_E^\ast=\mathcal K_{\mathrm{log}}\), there is no
oracle--non-oracle ordering to resolve. Otherwise, define
\[
\gamma_E
:=
\min_{k\in
\mathcal K_{\mathrm{log}}\setminus\mathcal K_E^\ast}
\left(E_k^2-E_\ast^2\right)
>0.
\]
For any \(s\in\mathcal K_E^\ast\) and
\(k\in\mathcal K_{\mathrm{log}}\setminus\mathcal K_E^\ast\), the orthogonal
decomposition gives
\[
\eta_{m,k}^2-\eta_{m,s}^2
\geq
\gamma_E
-
\max_{j\in\mathcal K_{\mathrm{log}}}\rho_{m,j}^2.
\]
Since the archive is finite and \(\rho_{m,j}\to0\) for every fixed \(j\),
\[
\max_{j\in\mathcal K_{\mathrm{log}}}\rho_{m,j}^2
\longrightarrow0.
\]
Hence, for all sufficiently large \(m\),
\[
\eta_{m,s}<\eta_{m,k}
\qquad
\forall s\in\mathcal K_E^\ast,\quad
\forall k\in
\mathcal K_{\mathrm{log}}\setminus\mathcal K_E^\ast.
\]
Thus, conforming refinement makes the unresolved projection defects uniformly
too small to offset the oracle--non-oracle energy gap, eliminating the
finite-resolution rank-reversal mechanism.
\end{remark}

The finite-gap criterion above is not directly verifiable at finite
resolution. Both \(\varepsilon_m\) and \(\delta_E\) depend on the inaccessible
exact energy errors \(E_k\), and no computable convergence rate for
\(\varepsilon_m\) is assumed. Hence, observing the same monitor
minimizer over finitely many auxiliary levels does not imply that
\(\varepsilon_m<\delta_E\); the ordering obstruction may disappear only under
further refinement.

Such stabilization nevertheless provides an empirical indication of selection
robustness under the tested auxiliary refinements. In experiments with
reference information, the ratios \(\eta_{m,k}/E_k\) quantify how closely
the monitors recover the corresponding energy errors. Without reference
information, however, stabilization alone is not a certificate, motivating
the computable finite-resolution guarantees derived next from the hierarchical
lower--upper intervals.

\subsection{Finite-resolution near-oracle control and logged-oracle certification}
\label{subsec:finite_oracle_guarantees}

Throughout this subsection, the finite archive
\(\mathcal K_{\mathrm{log}}\) is prescribed before selection and remains
fixed under auxiliary refinement. Changing the archive changes the target
oracle and therefore defines a distinct selection problem. Write
\[
E_{\mathrm{log}}^\ast
:=
\min_{k\in\mathcal K_{\mathrm{log}}}E_k.
\]

Fix a nested auxiliary pair
\[
V_m\subset V_M
\]
and a prescribed factor \(q\in(0,1)\). The lower-monitor selection rule is
independent of \(q\), which enters only through the conditional upper
estimates. For each \(k\in\mathcal K_{\mathrm{log}}\), define
\[
\mathcal U_{m,M,q,k}
:=
\mathcal U_{m,M,q}(u_k).
\]
Whenever saturation holds at checkpoint \(k\),
\[
\eta_{m,k}
\leq
E_k
\leq
\mathcal U_{m,M,q,k}.
\]

\begin{proposition}[Finite-resolution near-oracle control]
\label{prop:finite_near_oracle}
Let
\[
k_m^\ast\in\mathcal K_m^\ast
\]
be selected by the lower monitor. Assume that saturation holds at
\(k_m^\ast\) and that
\[
\eta_{m,k_m^\ast}>0.
\]
Then
\[
\frac{E_{k_m^\ast}}{E_{\mathrm{log}}^\ast}
\leq
\frac{\mathcal U_{m,M,q,k_m^\ast}}
{\eta_{m,k_m^\ast}}.
\]
\end{proposition}

\begin{proof}
Let \(k_E^\ast\in\mathcal K_E^\ast\). By monitor minimality and the
lower-bound property,
\[
\eta_{m,k_m^\ast}
\leq
\eta_{m,k_E^\ast}
\leq
E_{\mathrm{log}}^\ast.
\]
Since \(\eta_{m,k_m^\ast}>0\), one has
\(E_{\mathrm{log}}^\ast>0\). Saturation at \(k_m^\ast\) gives
\[
E_{k_m^\ast}
\leq
\mathcal U_{m,M,q,k_m^\ast}.
\]
Dividing the two bounds yields the result.
\end{proof}

\begin{corollary}[Exact logged-oracle certification by interval separation]
\label{cor:interval_separation_oracle}
Let \(s\in\mathcal K_{\mathrm{log}}\), and assume that saturation holds at
\(s\). If
\[
\mathcal U_{m,M,q,s}
<
\min_{\substack{k\in\mathcal K_{\mathrm{log}}\\ k\neq s}}
\eta_{m,k},
\]
then \(s\) is the unique logged energy-oracle checkpoint:
\[
E_s<E_k
\qquad
\forall k\in\mathcal K_{\mathrm{log}},\quad k\neq s.
\]
\end{corollary}

\begin{proof}
For every \(k\in\mathcal K_{\mathrm{log}}\setminus\{s\}\),
\[
E_s
\leq
\mathcal U_{m,M,q,s}
<
\eta_{m,k}
\leq
E_k.
\]
\end{proof}

\paragraph{Operational selection procedure}

At a given auxiliary resolution, evaluate \(\eta_{m,k}\) over the prescribed
archive and select
\[
k_m^\ast
\in
\operatorname*{arg\,min}_{k\in\mathcal K_{\mathrm{log}}}
\eta_{m,k}.
\]
The hierarchical upper estimate
\(\mathcal U_{m,M,q,k_m^\ast}\) is then computed at the selected checkpoint.
Under saturation at \(k_m^\ast\), the separation condition
\[
\mathcal U_{m,M,q,k_m^\ast}
<
\min_{\substack{k\in\mathcal K_{\mathrm{log}}\\ k\neq k_m^\ast}}
\eta_{m,k}
\]
certifies \(k_m^\ast\) as the unique logged energy oracle.

If separation is not attained and
\(\eta_{m,k_m^\ast}>0\), the same saturation hypothesis yields the computable
factor
\[
\Gamma_{\mathrm{log},k_m^\ast}
:=
\frac{\mathcal U_{m,M,q,k_m^\ast}}
{\eta_{m,k_m^\ast}},
\]
with
\[
\frac{E_{k_m^\ast}}{E_{\mathrm{log}}^\ast}
\leq
\Gamma_{\mathrm{log},k_m^\ast}.
\]
If this bound is sufficiently sharp for the intended use,
\(k_m^\ast\) may be retained; otherwise, the lower auxiliary space is refined,
the monitors are reevaluated over the same archive, and the procedure is
repeated with a further nested enrichment.

\paragraph{Scope and contribution}

In the present setting, selection is restricted to candidates generated by the
optimization run and retained in the prescribed archive; it cannot compensate
for a trajectory that fails to contain a sufficiently accurate approximation.
The symmetric coercive Ritz structure nevertheless yields an important
alignment:
\[
J(u_k)-J(u^\ast)=\tfrac12 E_k^2.
\]
Hence the exact Ritz functional and the energy error induce the same ordering
over admissible candidates. Practical empirical or quadrature-based evaluations
may, however, perturb this ordering unless the associated approximation errors
are sufficiently controlled; see the discussion following
Proposition~\ref{prop:best_approximation}.

The contribution is therefore to recover, without a reference solution,
the minimum-energy-error candidate contained in the prescribed archive.
Conforming refinement yields eventual logged-oracle recovery, while the
hierarchical bounds provide finite-resolution near-oracle control and, under
saturation, exact certification by interval separation. Whether the archive
omits a better iterate from a denser computed trajectory is the distinct
coverage question addressed next.

\section{Logging resolution and trajectory-wide guarantees}
\label{sec:trajectory_wide_control}

A logging schedule determines the finite archive over which checkpoint
selection is performed. Consequently, the corresponding logged oracle depends
on the logging resolution: a coarser archive may exclude lower-energy iterates
and thereby increase the smallest attainable energy error among the retained
checkpoints.

For a uniform logging stride \(h\) dividing \(T\), let
\[
\mathcal K_h:=\{0,h,2h,\ldots,T\},
\qquad
E_h^\ast:=\min_{k\in\mathcal K_h}E_k.
\]
If \(h_2\) is a multiple of \(h_1\), then
\[
\mathcal K_{h_2}\subseteq\mathcal K_{h_1},
\qquad
E_{h_1}^\ast\leq E_{h_2}^\ast.
\]
Thus, coarser logging cannot improve the best energy-error level available
within the archive.

To quantify this coverage effect, let
\(\mathcal K_{\mathrm{traj}}\) be a prescribed finite comparison trajectory,
which may in particular consist of a denser set of iterates from the same
optimization run, and fix an archive
\[
\mathcal K_{\mathrm{log}}
\subseteq
\mathcal K_{\mathrm{traj}}.
\]
Define
\[
E_{\mathrm{log}}^\ast
:=
\min_{k\in\mathcal K_{\mathrm{log}}}E_k,
\qquad
E_{\mathrm{traj}}^\ast
:=
\min_{j\in\mathcal K_{\mathrm{traj}}}E_j.
\]
Then
\[
0
\leq
E_{\mathrm{log}}^\ast-E_{\mathrm{traj}}^\ast,
\]
and this difference measures the loss induced by archive subsampling relative
to the comparison trajectory.

\subsection{Competitive archive coverage and trajectory-wide control}
\label{subsec:archive_resolution}

Assume that \(\eta_{m,j}\) is available for every
\(j\in\mathcal K_{\mathrm{traj}}\), and select
\[
s
\in
\operatorname*{arg\,min}_{k\in\mathcal K_{\mathrm{log}}}
\eta_{m,k}.
\]
Fix a nested pair \(V_m\subset V_M\) and a prescribed
\(q\in(0,1)\), and assume that saturation holds at \(s\). Define
\[
U_s:=\mathcal U_{m,M,q,s},
\qquad
E_s\leq U_s.
\]

The available lower--upper information allows trajectory iterates that cannot
improve upon \(s\) to be excluded. Define
\begin{equation}
\mathcal K_{\mathrm{comp},s}
:=
\left\{
j\in\mathcal K_{\mathrm{traj}}
:
\eta_{m,j}\leq U_s
\right\}.
\label{eq:competitive_trajectory_set}
\end{equation}
Indeed, if \(j\notin\mathcal K_{\mathrm{comp},s}\), then
\[
E_j
\geq
\eta_{m,j}
>
U_s
\geq
E_s,
\]
so \(j\) cannot improve upon the selected checkpoint. Conversely, if
\(j^\ast\in\operatorname*{arg\,min}_{j\in\mathcal K_{\mathrm{traj}}}E_j\)
is a trajectory oracle, then
\[
\eta_{m,j^\ast}
\leq
E_{\mathrm{traj}}^\ast
\leq
E_s
\leq
U_s.
\]
Hence every trajectory oracle belongs to
\(\mathcal K_{\mathrm{comp},s}\), and trajectory-wide coverage may be
restricted to this competitive set.

Define the competitive coverage radius by
\begin{equation}
r_{\mathrm{comp},s}
:=
\max_{j\in\mathcal K_{\mathrm{comp},s}}
\min_{k\in\mathcal K_{\mathrm{log}}}
\lVert u_j-u_k\rVert_a.
\label{eq:competitive_coverage_radius}
\end{equation}
It measures the largest energy distance from a competitive trajectory iterate
to its nearest retained checkpoint. These pairwise distances do not require
access to \(u^\ast\).

\begin{proposition}[Trajectory-wide control through competitive coverage]
\label{prop:competitive_trajectory_control}
Under the preceding assumptions,
\begin{equation}
0
\leq
E_{\mathrm{log}}^\ast-E_{\mathrm{traj}}^\ast
\leq
r_{\mathrm{comp},s}.
\label{eq:archive_subsampling_bound}
\end{equation}
Moreover,
\begin{equation}
0
\leq
E_s-E_{\mathrm{traj}}^\ast
\leq
\underbrace{U_s-\eta_{m,s}}_{
\substack{\text{auxiliary-resolution}\\\text{gap}}
}
+
\underbrace{r_{\mathrm{comp},s}}_{
\substack{\text{archive-coverage}\\\text{loss}}
}.
\label{eq:trajectory_additive_bound}
\end{equation}
The trajectory-oracle level satisfies
\begin{equation}
E_{\mathrm{traj}}^\ast
\geq
\eta_{m,s}-r_{\mathrm{comp},s}.
\label{eq:trajectory_oracle_lower_bound}
\end{equation}
Hence, if
\[
r_{\mathrm{comp},s}<\eta_{m,s},
\]
then
\begin{equation}
\Gamma_{\mathrm{traj},s}
:=
\frac{U_s}
{\eta_{m,s}-r_{\mathrm{comp},s}}
\label{eq:trajectory_near_oracle_factor}
\end{equation}
is well defined and satisfies
\begin{equation}
\frac{E_s}{E_{\mathrm{traj}}^\ast}
\leq
\Gamma_{\mathrm{traj},s}.
\label{eq:trajectory_near_oracle_bound}
\end{equation}

If, in addition,
\begin{equation}
U_s
<
\min_{k\in\mathcal K_{\mathrm{log}}\setminus\{s\}}
\eta_{m,k},
\label{eq:trajectory_logged_interval_separation}
\end{equation}
then \(s\) is the unique logged energy oracle and
\begin{equation}
E_s
=
E_{\mathrm{log}}^\ast,
\qquad
0
\leq
E_s-E_{\mathrm{traj}}^\ast
\leq
r_{\mathrm{comp},s}.
\label{eq:certified_trajectory_coverage_bound}
\end{equation}
\end{proposition}

\begin{proof}
Let \(j^\ast\in\mathcal K_{\mathrm{traj}}\) attain
\(E_{\mathrm{traj}}^\ast\). Since
\(j^\ast\in\mathcal K_{\mathrm{comp},s}\), choose
\(k^\ast\in\mathcal K_{\mathrm{log}}\) such that
\[
\lVert u_{j^\ast}-u_{k^\ast}\rVert_a
\leq
r_{\mathrm{comp},s}.
\]
The triangle inequality gives
\[
E_{k^\ast}
\leq
E_{\mathrm{traj}}^\ast+r_{\mathrm{comp},s}.
\]
Since
\(E_{\mathrm{traj}}^\ast\leq E_{\mathrm{log}}^\ast\leq E_{k^\ast}\),
\eqref{eq:archive_subsampling_bound} follows.

Let \(k_{\mathrm{log}}^\ast\) attain
\(E_{\mathrm{log}}^\ast\). Monitor minimality and the lower-bound property give
\[
\eta_{m,s}
\leq
\eta_{m,k_{\mathrm{log}}^\ast}
\leq
E_{\mathrm{log}}^\ast.
\]
Together with \eqref{eq:archive_subsampling_bound}, this yields
\[
E_{\mathrm{traj}}^\ast
\geq
\eta_{m,s}-r_{\mathrm{comp},s},
\]
proving \eqref{eq:trajectory_oracle_lower_bound}. Since \(E_s\leq U_s\),
\[
E_s-E_{\mathrm{traj}}^\ast
\leq
U_s-\eta_{m,s}+r_{\mathrm{comp},s},
\]
which gives \eqref{eq:trajectory_additive_bound}. If
\(r_{\mathrm{comp},s}<\eta_{m,s}\), division by the positive lower bound in
\eqref{eq:trajectory_oracle_lower_bound} yields
\eqref{eq:trajectory_near_oracle_bound}.

Finally, \eqref{eq:trajectory_logged_interval_separation} and
Corollary~\ref{cor:interval_separation_oracle} certify \(s\) as the unique
logged energy oracle. Hence \(E_s=E_{\mathrm{log}}^\ast\), and
\eqref{eq:certified_trajectory_coverage_bound} follows from
\eqref{eq:archive_subsampling_bound}.
\end{proof}

The trajectory-wide suboptimality admits the exact decomposition
\[
E_s-E_{\mathrm{traj}}^\ast
=
\bigl(E_s-E_{\mathrm{log}}^\ast\bigr)
+
\bigl(E_{\mathrm{log}}^\ast-E_{\mathrm{traj}}^\ast\bigr).
\]
The first term is the loss incurred by monitor-based selection within the
prescribed archive, whereas the second is the loss caused by archive
subsampling relative to the comparison trajectory. Proposition~\ref{prop:competitive_trajectory_control} bounds these two contributions by
\[
0
\leq
E_s-E_{\mathrm{log}}^\ast
\leq
U_s-\eta_{m,s},
\qquad
0
\leq
E_{\mathrm{log}}^\ast-E_{\mathrm{traj}}^\ast
\leq
r_{\mathrm{comp},s}.
\]
Hence \eqref{eq:trajectory_additive_bound} separates the trajectory-wide
suboptimality into an auxiliary-resolution contribution and an
archive-coverage contribution.

\subsection{Competitive coverage evaluation and path-length bounds}
\label{subsec:coverage_estimates}

When the required pairwise energy distances are available, the competitive
coverage radius can be evaluated directly. Define
\[
d_a\bigl(j,\mathcal K_{\mathrm{log}}\bigr)
:=
\min_{k\in\mathcal K_{\mathrm{log}}}
\lVert u_j-u_k\rVert_a.
\]
Then
\[
r_{\mathrm{comp},s}
=
\max_{j\in\mathcal K_{\mathrm{comp},s}}
d_a\bigl(j,\mathcal K_{\mathrm{log}}\bigr).
\]

Direct evaluation may, however, require many pairwise energy-distance
computations when the comparison trajectory or retained archive is large.
Successive trajectory increments provide an alternative based only on local
inter-iterate distances or their upper bounds, for example when these are
available from densely recorded consecutive checkpoints or from bounds on
successive iterate displacements.

Suppose that
\[
\mathcal K_{\mathrm{traj}}=\{0,\ldots,T\},
\qquad
\mathcal K_{\mathrm{log}}=\{t_0,\ldots,t_N\},
\qquad
0=t_0<\cdots<t_N=T,
\]
and that
\[
\lVert u_{j+1}-u_j\rVert_a
\leq
\overline d_j,
\qquad
j=0,\ldots,T-1.
\]

\begin{corollary}[Competitive path-length bound]
\label{cor:competitive_path_length}
Under the preceding assumptions,
\begin{equation}
r_{\mathrm{comp},s}
\leq
\overline r_{\mathrm{comp},s},
\label{eq:competitive_path_length_bound}
\end{equation}
where
\begin{equation}
\overline r_{\mathrm{comp},s}
:=
\max_{\substack{0\leq i<N\\
j\in\mathcal K_{\mathrm{comp},s}\cap\{t_i,\ldots,t_{i+1}\}}}
\min
\left\{
\sum_{\ell=t_i}^{j-1}\overline d_\ell,
\,
\sum_{\ell=j}^{t_{i+1}-1}\overline d_\ell
\right\},
\label{eq:competitive_path_length_radius}
\end{equation}
with empty sums interpreted as zero. For a uniform stride \(h\) and
\(\overline d_j\leq d_{\max}\),
\begin{equation}
r_{\mathrm{comp},s}
\leq
\left\lfloor\frac{h}{2}\right\rfloor d_{\max}.
\label{eq:uniform_stride_coverage_bound}
\end{equation}
\end{corollary}

\begin{proof}
For \(j\in\{t_i,\ldots,t_{i+1}\}\), the triangle inequality gives
\[
\lVert u_j-u_{t_i}\rVert_a
\leq
\sum_{\ell=t_i}^{j-1}\overline d_\ell,
\qquad
\lVert u_j-u_{t_{i+1}}\rVert_a
\leq
\sum_{\ell=j}^{t_{i+1}-1}\overline d_\ell.
\]
Since both endpoints are retained, the distance from \(u_j\) to the archive
is bounded by the smaller path length. Maximizing over the competitive set
gives \eqref{eq:competitive_path_length_bound}. For uniform stride \(h\),
every index is at most \(\lfloor h/2\rfloor\) increments from a retained
endpoint, yielding \eqref{eq:uniform_stride_coverage_bound}.
\end{proof}

The path-length construction yields the valid upper bound
\[
r_{\mathrm{comp},s}
\leq
\overline r_{\mathrm{comp},s},
\]
and, for uniform logging, makes the dependence on the stride \(h\) explicit
through \eqref{eq:uniform_stride_coverage_bound}. Direct evaluation of
\(r_{\mathrm{comp},s}\), when available, avoids this additional relaxation,
whereas the path-length bound may therefore be less sharp. More generally, any
valid upper bound
\(\widehat r_{\mathrm{comp},s}\geq r_{\mathrm{comp},s}\) may replace
\(r_{\mathrm{comp},s}\) in the trajectory-wide guarantees, yielding
correspondingly conservative bounds.

Thus, evaluating or bounding \(r_{\mathrm{comp},s}\) quantifies the price of
archive subsampling. This enables the logging stride to be chosen by balancing
sparsity against trajectory-wide accuracy.

\subsection{Logging-resolution design with trajectory-wide control}
\label{subsec:coverage_logging_design}

Coarser logging reduces the number of retained checkpoints and can make the
interval-separation condition \eqref{eq:trajectory_logged_interval_separation}
easier to satisfy by reducing the competitor set; however, it may also omit
lower-energy trajectory iterates. The objective is therefore to determine how
far the logging stride can be increased while keeping the resulting
trajectory-wide loss under quantitative control.

Let \(\mathcal H\) be a finite family of candidate uniform strides, chosen so
that the corresponding archives \((\mathcal K_h)_{h\in\mathcal H}\) are
nested, with larger \(h\) corresponding to sparser logging. All candidates are
assessed relative to the same prescribed comparison trajectory
\(\mathcal K_{\mathrm{traj}}\).

For each \(h\in\mathcal H\), select
\[
s_h
\in
\operatorname*{arg\,min}_{k\in\mathcal K_h}\eta_{m,k},
\]
and, assuming saturation at \(s_h\), define
\[
U_h:=\mathcal U_{m,M,q,s_h},
\qquad
a_h:=U_h-\eta_{m,s_h},
\qquad
r_h:=r_{\mathrm{comp},s_h},
\]
where \(r_h\) is evaluated relative to
\(\mathcal K_{\mathrm{log}}=\mathcal K_h\).
Proposition~\ref{prop:competitive_trajectory_control} gives
\[
0
\leq
E_h^\ast-E_{\mathrm{traj}}^\ast
\leq
r_h,
\qquad
0
\leq
E_{s_h}-E_{\mathrm{traj}}^\ast
\leq
a_h+r_h.
\]
Thus, \(r_h\) bounds the loss induced by archive coarsening, whereas
\(a_h+r_h\) bounds the trajectory-wide suboptimality of the checkpoint
selected from that archive. If only a valid upper bound
\(\widehat r_h\geq r_h\) is available, it may replace \(r_h\) in these
guarantees conservatively.

Whenever \(r_h<\eta_{m,s_h}\), the corresponding relative guarantee is
\[
\frac{E_{s_h}}{E_{\mathrm{traj}}^\ast}
\leq
\Gamma_{\mathrm{traj},s_h}
:=
\frac{U_h}{\eta_{m,s_h}-r_h}.
\]

For prescribed tolerances \(\tau_{\mathrm{log}}>0\) and
\(\tau_{\mathrm{tot}}>0\), the preceding bounds yield
\begin{equation}
h_{\mathrm{cov}}^\ast
:=
\max
\left\{
h\in\mathcal H:
r_h\leq\tau_{\mathrm{log}}
\right\},
\label{eq:coverage_stride_design}
\end{equation}
and
\begin{equation}
h_{\mathrm{sel}}^\ast
:=
\max
\left\{
h\in\mathcal H:
a_h+r_h\leq\tau_{\mathrm{tot}}
\right\}.
\label{eq:selected_stride_design}
\end{equation}
Whenever the corresponding feasible sets are nonempty, these are the sparsest
tested resolutions satisfying, respectively,
\[
E_{h_{\mathrm{cov}}^\ast}^\ast
\leq
E_{\mathrm{traj}}^\ast+\tau_{\mathrm{log}},
\qquad
E_{s_{h_{\mathrm{sel}}^\ast}}
\leq
E_{\mathrm{traj}}^\ast+\tau_{\mathrm{tot}}.
\]
A relative alternative is to retain the largest tested \(h\) satisfying
\[
\Gamma_{\mathrm{traj},s_h}\leq1+\varepsilon,
\]
which guarantees
\[
E_{s_h}
\leq
(1+\varepsilon)E_{\mathrm{traj}}^\ast.
\]

Candidate strides can be assessed by deterministic subsampling of a single
densely recorded comparison run, without retraining or modifying the
optimization schedule. The quantities \(r_h\), \(a_h+r_h\), and
\(\Gamma_{\mathrm{traj},s_h}\) quantify how far logging can be coarsened:
without a prescribed accuracy criterion they expose the sparsity--accuracy
trade-off, whereas with one they identify the sparsest tested resolution
satisfying it.

\subsection{Exact oracle inclusion and trajectory-oracle certification}
\label{subsec:exact_trajectory_certificates}

The preceding analysis quantifies the accuracy loss induced by archive
coarsening and thereby provides a basis for choosing the logging resolution.
A stronger question is whether a coarsened archive has incurred any loss at
all, that is, whether it still contains a trajectory oracle. We therefore
derive exact certificates for trajectory-oracle inclusion and, more strongly,
for the selected checkpoint itself to be a trajectory oracle. When
\(\mathcal K_{\mathrm{traj}}\) contains the full recorded optimization
trajectory, these certificates respectively show that sparse logging has
preserved, or that the selected checkpoint attains, the best energy-error
level reached along that trajectory.

To compare the best retained and omitted trajectory iterates, define
\[
L_h^{\mathrm{out}}
:=
\min_{j\in\mathcal K_{\mathrm{traj}}\setminus\mathcal K_h}
\eta_{m,j},
\]
with the convention \(L_h^{\mathrm{out}}=+\infty\) if
\(\mathcal K_h=\mathcal K_{\mathrm{traj}}\). When valid hierarchical upper
estimates \(U_k:=\mathcal U_{m,M,q,k}\) are available for all
\(k\in\mathcal K_h\), define
\[
B_h
:=
\min_{k\in\mathcal K_h}U_k.
\]
Then \(L_h^{\mathrm{out}}\) is a lower bound on the best energy error among
omitted trajectory iterates, whereas \(B_h\) is an upper bound on the best
energy error attained within the archive.

\begin{proposition}[Exact trajectory-oracle certificates]
\label{prop:exact_trajectory_oracle_certificates}
For a candidate archive \(\mathcal K_h\), the following statements hold.

\begin{enumerate}
\item If valid upper estimates \(U_k\) are available for all
\(k\in\mathcal K_h\) and
\begin{equation}
B_h<L_h^{\mathrm{out}},
\label{eq:exact_oracle_inclusion}
\end{equation}
then \(\mathcal K_h\) contains a trajectory oracle, and hence
\begin{equation}
E_h^\ast=E_{\mathrm{traj}}^\ast.
\label{eq:exact_archive_oracle_recovery}
\end{equation}

\item If \(U_h\) is a valid upper estimate at \(s_h\) and
\begin{equation}
U_h
<
\min_{j\in\mathcal K_{\mathrm{traj}}\setminus\{s_h\}}
\eta_{m,j},
\label{eq:full_trajectory_interval_separation}
\end{equation}
then \(s_h\) is the unique trajectory oracle, so that
\begin{equation}
E_{s_h}
=
E_h^\ast
=
E_{\mathrm{traj}}^\ast.
\label{eq:full_trajectory_oracle_recovery}
\end{equation}
\end{enumerate}
\end{proposition}

\begin{proof}
Choose \(\bar k\in\mathcal K_h\) such that
\(U_{\bar k}=B_h\). Under \eqref{eq:exact_oracle_inclusion}, every
\(j\in\mathcal K_{\mathrm{traj}}\setminus\mathcal K_h\) satisfies
\[
E_{\bar k}
\leq
U_{\bar k}
=
B_h
<
L_h^{\mathrm{out}}
\leq
\eta_{m,j}
\leq
E_j.
\]
Thus no omitted checkpoint can attain the trajectory minimum, proving
\eqref{eq:exact_archive_oracle_recovery}.

Under \eqref{eq:full_trajectory_interval_separation},
\[
E_{s_h}
\leq
U_h
<
\eta_{m,j}
\leq
E_j
\qquad
\forall j\in
\mathcal K_{\mathrm{traj}}\setminus\{s_h\}.
\]
Hence \(s_h\) is the unique trajectory oracle, which yields
\eqref{eq:full_trajectory_oracle_recovery}.
\end{proof}

For nested candidate archives assessed with the same valid checkpointwise
upper estimates, \(B_h\) cannot increase under archive refinement, whereas
\(L_h^{\mathrm{out}}\) cannot decrease. Hence
\eqref{eq:exact_oracle_inclusion}, once satisfied, remains valid under further
refinement. The sparsest tested archive certified to contain a trajectory
oracle is therefore
\begin{equation}
h_{\mathrm{incl}}^\ast
:=
\max
\left\{
h\in\mathcal H:
B_h<L_h^{\mathrm{out}}
\right\},
\label{eq:oracle_inclusion_stride}
\end{equation}
whenever the feasible set is nonempty.

If a valid upper estimate is available only at the selected checkpoint, the
selected-only condition
\[
U_h<L_h^{\mathrm{out}}
\]
still certifies
\[
E_h^\ast=E_{\mathrm{traj}}^\ast.
\]

Likewise, the sparsest tested archive for which the selected checkpoint is
certified as the unique trajectory oracle is
\begin{equation}
h_{\mathrm{cert}}^\ast
:=
\max
\left\{
h\in\mathcal H:
U_h
<
\min_{j\in\mathcal K_{\mathrm{traj}}\setminus\{s_h\}}
\eta_{m,j}
\right\},
\label{eq:certified_oracle_stride}
\end{equation}
whenever the feasible set is nonempty.

Together, these criteria yield a hierarchy from quantitative control to exact
certification: \(r_h\leq\tau_{\mathrm{log}}\) controls the loss caused by
archive coarsening, \(a_h+r_h\leq\tau_{\mathrm{tot}}\) controls the
trajectory-wide suboptimality of the selected checkpoint,
\(B_h<L_h^{\mathrm{out}}\) certifies preservation of a trajectory oracle, and
\eqref{eq:full_trajectory_interval_separation} certifies the selected
checkpoint itself as the unique trajectory oracle.

\subsection{Joint logging and auxiliary refinement}
\label{subsec:joint_trajectory_refinement}

The preceding results provide finite-resolution guarantees and exact
certificates. We now consider the complementary asymptotic regime in which
logging and auxiliary resolutions are refined simultaneously.

\begin{corollary}[Trajectory-oracle recovery under joint refinement]
\label{cor:joint_trajectory_refinement}
Consider a sequence of logging and auxiliary configurations, all assessed
relative to the same finite comparison trajectory. Let \(s_n\) be the selected
checkpoint, let \(U_n\) be a valid upper estimate at \(s_n\), and define
\[
a_n:=U_n-\eta_{m_n,s_n},
\qquad
r_n:=r_{\mathrm{comp},s_n}.
\]
If
\[
a_n\longrightarrow0,
\qquad
r_n\longrightarrow0,
\]
then
\[
E_{s_n}\longrightarrow E_{\mathrm{traj}}^\ast.
\]
If \(E_{\mathrm{traj}}^\ast>0\), then
\(\Gamma_{\mathrm{traj},s_n}\) is well defined for all sufficiently large \(n\)
and
\[
\Gamma_{\mathrm{traj},s_n}\longrightarrow1.
\]
\end{corollary}

\begin{proof}
Proposition~\ref{prop:competitive_trajectory_control} gives
\[
0
\leq
E_{s_n}-E_{\mathrm{traj}}^\ast
\leq
a_n+r_n,
\]
so \(E_{s_n}\to E_{\mathrm{traj}}^\ast\). Moreover,
\[
\eta_{m_n,s_n}
\leq
E_{s_n}
\leq
U_n,
\qquad
U_n-\eta_{m_n,s_n}=a_n,
\]
hence
\[
\eta_{m_n,s_n}
\longrightarrow
E_{\mathrm{traj}}^\ast,
\qquad
U_n
\longrightarrow
E_{\mathrm{traj}}^\ast.
\]
If \(E_{\mathrm{traj}}^\ast>0\), then
\[
\eta_{m_n,s_n}-r_n
\longrightarrow
E_{\mathrm{traj}}^\ast>0,
\]
and therefore
\[
\Gamma_{\mathrm{traj},s_n}
=
\frac{U_n}{\eta_{m_n,s_n}-r_n}
\longrightarrow1.
\]
\end{proof}

Corollary~\ref{cor:joint_trajectory_refinement} yields trajectory-oracle
consistency for the complete selection procedure relative to the prescribed
comparison trajectory. The auxiliary-resolution gap \(a_n\) controls selection
within the logged archive, while \(r_n\) controls the loss due to archive
coverage; their joint decay is sufficient for the selected checkpoint to
recover the trajectory-oracle level asymptotically.

\section{Numerical validation}
\label{sec:numerical}

This section evaluates the conforming Riesz monitor \(\eta_m\) as a
post-training error-assessment and checkpoint-selection criterion. The
experiments examine energy-scale calibration, reference-free oracle selection,
auxiliary refinement, logging resolution and trajectory coverage, and
applicability to a non-manufactured benchmark. Reference solutions are used
only for validation and enter neither training, Riesz reconstruction,
checkpoint selection, nor certification criteria.

We denote by
\[
E_{\rm ref}
=
\begin{cases}
E_{\rm quad}:=\|u_\theta-u^\ast\|_{a,{\rm quad}},
& \text{manufactured benchmarks},\\[1mm]
E_{\rm FEM}:=\|u_\theta-u_{\rm FEM}^{\rm ref}\|_a,
& \text{perforated plate},
\end{cases}
\]
the numerical validation energy error. Thus \(\eta_m/E_{\rm ref}\) measures
monitor calibration. For manufactured problems it approximates the fraction
of the exact energy error resolved by the conforming reconstruction; for the
perforated plate it provides an external calibration ratio against the
independent refined FEM reference.

We use the hierarchical gap \(\delta_{m,M}\) and conditional upper estimate
\(\Ucal_{m,M,q}\) from
Subsection~\ref{subsec:conforming_riesz_hierarchy}, with \(q=0.9\) unless
stated otherwise. Monitor-based selection is independent of \(q\); sensitivity
of the conditional certificates and upper estimates to \(q\) is reported
in~\ref{app:q_sensitivity}. For validation and refinement diagnostics
we also use
\[
q_{\rm obs}
:=
\left(
\frac{E_{\rm quad}^2-\eta_M^2}
     {E_{\rm quad}^2-\eta_m^2}
\right)^{1/2},
\qquad
\gamma_{m,M,L}
:=
\frac{\delta_{M,L}}{\delta_{m,M}},
\qquad
D_c
:=
\frac{E_{{\rm ref},k_c}}
     {E_{{\rm ref},k_{\rm oracle}}},
\]
whenever defined. Here \(q_{\rm obs}\le q\) checks saturation against the
manufactured reference at the tested checkpoint, while
\(\gamma_{m,M,L}<1\) indicates decay of successive newly resolved Riesz
components but does not itself verify saturation. The factor \(D_c\) measures
the energy deterioration of criterion \(c\) relative to the oracle of the same
logged archive; \(D_c=1\) denotes oracle-level selection with respect to the
validation reference. When \(E_{\rm quad}\) is available, we also report
\(\rho_{m,{\rm rel}}^{\rm quad}
:=(E_{\rm quad}^2-\eta_m^2)^{1/2}/E_{\rm quad}\)
as an external measure of the unresolved Riesz component. The experimental
protocols are summarized in Table~\ref{tab:experimental_protocol}.

\begin{table}[!t]
\centering
\small
\setlength{\tabcolsep}{4pt}
\caption{Experimental protocols. All networks use hyperbolic tangent
activations and Adam with learning rate \(10^{-3}\); each run uses a fixed
independent validation set of \(2\times10^4\) points for \(J_{\rm val}\), and
each benchmark is repeated over five seeds \(0,\ldots,4\). The notation
\(64^4\) denotes four hidden layers of width \(64\), and \(M_{\rm op}\)
denotes the operational auxiliary level.}
\label{tab:experimental_protocol}
\begin{adjustbox}{max width=\textwidth}
\begin{tabular}{lcccccc}
\toprule
Family
& Network
& Steps
& Log
& Training points
& Validation reference
& Auxiliary hierarchy \\
\midrule
Scalar diffusion
& \(2\!-\!64^4\!-\!1\)
& \(2000\)
& \(50\)
& \(10^4\)
& \(E_{\rm quad}\): \(96^2\), \(3{\times}3\) Gauss
& \(6,12,24,48\), \(M_{\rm op}=24\) \\
Elasticity
& \(2\!-\!80^4\!-\!2\)
& \(3000\)
& \(100\)
& \(12000\)
& \(E_{\rm quad}\): \(96^2\), \(3{\times}3\) Gauss
& \(6,12,24,48\), \(M_{\rm op}=24\) \\
L-shaped domain
& \(2\!-\!80^4\!-\!1\)
& \(4000\)
& \(100\)
& \(15000\)
& \(E_{\rm quad}\): \(128^2\), \(3{\times}3\) Gauss
& \(8,16,32,64,128\), \(M_{\rm op}=64\) \\
Perforated plate
& \(2\!-\!80^4\!-\!2\)
& \(4000\)
& \(100\)
& \(15000\)
& FEM \(256\) (check: \(192\))
& \(24,48,96\), \(M_{\rm op}=48\) \\
\bottomrule
\end{tabular}
\end{adjustbox}
\end{table}

Unless stated otherwise, reported values of the form \(x\pm y\) denote the
mean and one sample standard deviation over the five seeds.

For the manufactured square-domain tests, the auxiliary spaces are the
conforming \(Q_1\) spaces
\[
V_m^{Q_1}
=
\left\{
v_m\in C^0(\overline\Omega)\cap H_0^1(\Omega):
\ v_m|_K\in Q_1(K)\ \ \forall K\in\mathcal T_m
\right\},
\]
with \([V_m^{Q_1}]^2\) in elasticity; the L-shaped benchmark uses the same
construction on its active quadrilateral cells. To realize
Algorithm~\ref{alg:riesz_monitor} on the nested auxiliary hierarchies, we use
the common finest-level Galerkin form
\(G_m=P_{m,L}^{\top}G_LP_{m,L}\) and
\(r_m=P_{m,L}^{\top}r_L\), so nestedness and monitor monotonicity hold
algebraically at the assembled discrete level. The finest levels are \(L=48\)
for the standard scalar and elasticity hierarchies, \(L=96\) for the
additional \(\kappa_4\) certification hierarchy, and \(L=128\) for the
L-shaped hierarchy.

For the perforated plate, let \(V_{m,\circ}^{P_1}\) denote the conforming
piecewise-linear space satisfying the homogeneous essential condition on the
outer boundary. We use the exactly nested vector-valued hierarchy
\[
[V_{24,\circ}^{P_1}]^2
\subset
[V_{48,\circ}^{P_1}]^2
\subset
[V_{96,\circ}^{P_1}]^2,
\]
obtained by straight-sided uniform refinement of a common conforming triangular
mesh fitted to the fixed polygonal hole. All levels therefore represent the
same polygonal domain. The discrete-hierarchy and quadrature-order audit is
reported in~\ref{app:discrete_quadrature_audit}.
All approximations are trained with the Deep Ritz energy \cite{EYu2018}, with
essential boundary conditions imposed strongly. \(J_{\rm train}\) and
\(J_{\rm val}\) denote energy evaluations on the fixed training and independent
validation sets, respectively. The strong-form residual is used only as an
\(L^2\)-type baseline. Neural training and automatic differentiation use
single precision, whereas auxiliary assembly, Riesz solves, reference
quadratures, and FEM computations use double precision. The operational levels \(M_{\rm op}\) in
Table~\ref{tab:experimental_protocol} are those used for checkpoint ranking;
enriched levels are used only for post-training refinement, calibration, and
certification diagnostics.

\subsection{Manufactured-reference validation benchmarks}
\label{subsec:exact_reference_benchmarks}

\subsubsection{Scalar variable-coefficient diffusion}
\label{subsubsec:scalar_results}

We first consider
\[
-\nabla\cdot(\kappa\nabla u)=f\quad\text{in }\Omega=(0,1)^2,
\qquad
u=0\quad\text{on }\partial\Omega,
\]
with \(V_0=H_0^1(\Omega)\) and
\(a_i(u,v)=\int_\Omega \kappa_i\nabla u\cdot\nabla v\,dx\). The first three cases use the common manufactured solution
\[
u_i^\ast(x,y)=\sin(\pi x)\sin(\pi y),
\qquad i=1,2,3,
\]
with
\[
\kappa_1=1+\tfrac12\sin(2\pi x)\sin(2\pi y),\qquad
\kappa_2=1+\tfrac14\sin(6\pi x)\sin(6\pi y),
\]
\[
\kappa_3=1+0.45\sin(10\pi x)\sin(10\pi y).
\]
The fourth combines the more oscillatory coefficient
\[
\kappa_4=1+0.49\sin(14\pi x)\sin(14\pi y)
\]
with
\[
\begin{aligned}
u_4^\ast(x,y)
&=\sin(\pi x)\sin(\pi y)
+0.35\,x(1-x)y(1-y)\\
&\qquad\times
\exp\!\left[-80\big((x-0.35)^2+(y-0.65)^2\big)\right]
\left(1+0.5\sin(2\pi x)\sin(3\pi y)\right).
\end{aligned}
\]
In each case,
\[
f_i=-\nabla\cdot(\kappa_i\nabla u_i^\ast).
\]

\begin{figure}[!t]
\centering
\begin{subfigure}[t]{0.49\textwidth}
\centering
\includegraphics[width=\linewidth]{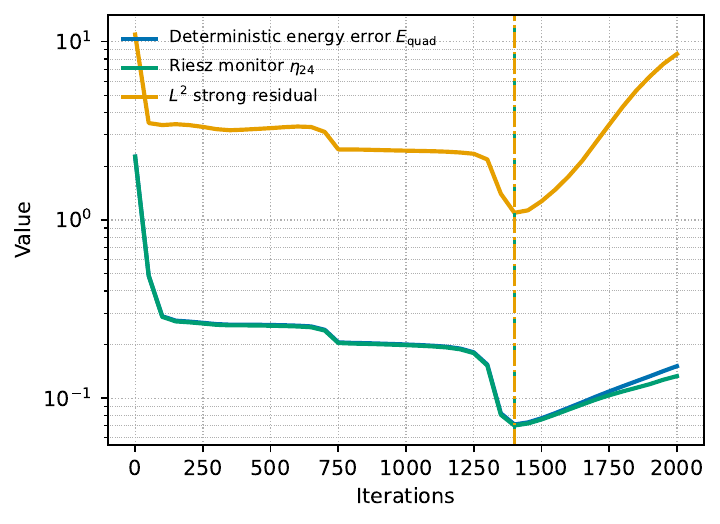}
\caption{Training diagnostics.}
\label{fig:scalar_training:a}
\end{subfigure}\hfill
\begin{subfigure}[t]{0.49\textwidth}
\centering
\includegraphics[width=\linewidth]{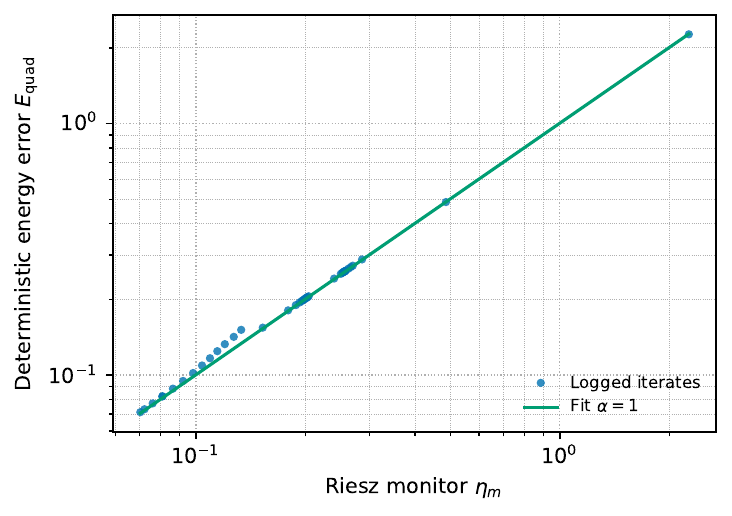}
\caption{Calibration against \(E_{\rm quad}\).}
\label{fig:scalar_training:b}
\end{subfigure}
\caption{Representative \(\kappa_1\) run (seed \(0\)): temporal tracking and
calibration of the operational monitor \(\eta_{24}\) against
\(E_{\rm quad}\).}
\label{fig:scalar_training}
\end{figure}

Figure~\ref{fig:scalar_training} shows that \(\eta_{24}\) tracks the energy
error on its natural scale and attains its minimum in the same late-training
regime. The strong \(L^2\) residual evolves on a different scale, as expected
from Remark~\ref{rem:strong_residual}.

\begin{table}[!t]
\centering
\scriptsize
\setlength{\tabcolsep}{3.4pt}
\renewcommand{\arraystretch}{1.12}
\caption{Scalar diffusion: checkpoint selection over five seeds.}
\label{tab:scalar_checkpoint_selection}
\begin{adjustbox}{max width=\textwidth}
\begin{tabular}{lcccccccc}
\toprule
\multirow{2}{*}{Selection criterion}
& \multicolumn{2}{c}{\(\kappa_1\)}
& \multicolumn{2}{c}{\(\kappa_2\)}
& \multicolumn{2}{c}{\(\kappa_3\)}
& \multicolumn{2}{c}{\(\kappa_4\)} \\
\cmidrule(lr){2-3}
\cmidrule(lr){4-5}
\cmidrule(lr){6-7}
\cmidrule(lr){8-9}
& Step & \(D_c\)
& Step & \(D_c\)
& Step & \(D_c\)
& Step & \(D_c\) \\
\midrule
Energy oracle \(\min E_{\rm quad}\)
& \(1390 \pm 201\) & \(1.000 \pm 0.000\)
& \(1420 \pm 175\) & \(1.000 \pm 0.000\)
& \(1410 \pm 222\) & \(1.000 \pm 0.000\)
& \(1480 \pm 289\) & \(1.000 \pm 0.000\) \\
Riesz monitor \(\min \eta_{24}\)
& \(1390 \pm 201\) & \(1.000 \pm 0.000\)
& \(1420 \pm 175\) & \(1.000 \pm 0.000\)
& \(1410 \pm 222\) & \(1.000 \pm 0.000\)
& \(1480 \pm 289\) & \(1.000 \pm 0.000\) \\
Strong residual \(L^2\)
& \(1390 \pm 201\) & \(1.000 \pm 0.000\)
& \(1250 \pm 332\) & \(1.105 \pm 0.223\)
& \(1230 \pm 355\) & \(1.074 \pm 0.163\)
& \(1300 \pm 237\) & \(1.049 \pm 0.109\) \\
Training Ritz loss \(\min J_{\rm train}\)
& \(2000 \pm 0\) & \(1.723 \pm 0.368\)
& \(2000 \pm 0\) & \(1.778 \pm 0.294\)
& \(2000 \pm 0\) & \(1.603 \pm 0.389\)
& \(2000 \pm 0\) & \(1.529 \pm 0.358\) \\
Validation Ritz loss \(\min J_{\rm val}\)
& \(1370 \pm 182\) & \(1.036 \pm 0.067\)
& \(1390 \pm 156\) & \(1.020 \pm 0.029\)
& \(1420 \pm 160\) & \(1.030 \pm 0.040\)
& \(1460 \pm 268\) & \(1.012 \pm 0.020\) \\
\bottomrule
\end{tabular}
\end{adjustbox}
\end{table}

Across all twenty scalar runs (Table~\ref{tab:scalar_checkpoint_selection}),
\(\eta_{24}\) selects the same logged checkpoint as the
\(E_{\rm quad}\)-oracle. The validation Ritz loss remains
close to the oracle, whereas the training loss selects the final iterate and
the strong residual departs from the oracle in the three more demanding
cases. This does not contradict the exact Ritz--energy identity, which applies to the
continuous functional \(J\); the discrete evaluations \(J_{\rm train}\) and
\(J_{\rm val}\) may perturb its ordering, as discussed after
Proposition~\ref{prop:best_approximation}. These comparisons concern numerically evaluated
criteria. Exact Ritz values preserve energy-error
ordering, whereas both approximate Ritz values
and finite-resolution Riesz monitors may alter it.
The Riesz framework provides energy-error monitoring
with explicit refinement guarantees and, under
the stated assumptions, finite-resolution
selection certificates.

\begin{table}[!t]
\centering
\small
\setlength{\tabcolsep}{5pt}
\renewcommand{\arraystretch}{1.12}
\caption{Scalar diffusion: calibration and hierarchical diagnostics at the
\(\eta_{24}\)-selected checkpoint.}
\label{tab:scalar_summary}
\begin{adjustbox}{max width=\textwidth}
\begin{tabular}{lccccc}
\toprule
\multirow{2}{*}{Coefficient}
& \multicolumn{2}{c}{Monitor calibration}
& \multicolumn{3}{c}{Refinement / estimator diagnostics} \\
\cmidrule(lr){2-3}
\cmidrule(lr){4-6}
& \(\eta_{24}/E_{\rm quad}\)
& \(\eta_{48}/E_{\rm quad}\)
& \(\gamma_{12,24,48}\)
& \(q_{\rm obs}\)
& \(\Ucal_{24,48,0.9}/E_{\rm quad}\) \\
\midrule
\(\kappa_1\)
& \(0.981 \pm 0.007\)
& \(0.995 \pm 0.002\)
& \(0.519 \pm 0.004\)
& \(0.503 \pm 0.001\)
& \(1.052 \pm 0.018\) \\
\(\kappa_2\)
& \(0.975 \pm 0.011\)
& \(0.994 \pm 0.003\)
& \(0.535 \pm 0.019\)
& \(0.506 \pm 0.004\)
& \(1.070 \pm 0.029\) \\
\(\kappa_3\)
& \(0.983 \pm 0.008\)
& \(0.996 \pm 0.002\)
& \(0.528 \pm 0.013\)
& \(0.506 \pm 0.003\)
& \(1.048 \pm 0.020\) \\
\(\kappa_4\)
& \(0.979 \pm 0.013\)
& \(0.995 \pm 0.004\)
& \(0.525 \pm 0.028\)
& \(0.509 \pm 0.011\)
& \(1.057 \pm 0.033\) \\
\bottomrule
\end{tabular}
\end{adjustbox}
\end{table}

At the selected checkpoints (Table~\ref{tab:scalar_summary}),
\(\eta_{24}/E_{\rm quad}=0.975\)--\(0.983\), increasing to
\(\eta_{48}/E_{\rm quad}=0.994\)--\(0.996\). The observed saturation factors are
near \(0.5\), and the conditional upper estimates remain within about
\(5\%\)--\(7\%\) of the reference error for \(q=0.9\).

Thus, in these scalar tests, the level-\(24\) monitor already resolves the
archive ordering needed for oracle selection and incurs below-\(1\%\) observed
overhead (Table~\ref{tab:computational_cost_summary}), supporting its use as
the operational monitor; enrichment further sharpens energy-scale calibration.

Across all four coefficients, auxiliary refinement increases
\(\eta_m/E_{\rm quad}\) toward unity
(Figure~\ref{fig:manufactured_aux_refinement:a}), consistent with
Corollary~\ref{cor:conforming_refinement}.

Having established energy-scale calibration, we next test finite-resolution
certification of the selected archive minimizer.

\paragraph{Finite-level interval certification}
We use \(\kappa_4\) to test finite-resolution certification across auxiliary
resolutions and logging strides. The dense log is
\[
\mathcal K_{50}=\{0,50,\ldots,2000\},
\]
and the nested subarchives are
\(\mathcal K_h=\{0,h,2h,\ldots,2000\}\) for
\(h\in\{50,100,200,400\}\). For the monitor-selected checkpoint \(s_h\), let
\(U_h:=\Ucal_{m,M,q,s_h}\). We report the conditional near-oracle factor
\(\Gamma_{{\rm log},h}=U_h/\eta_{m,s_h}\); interval separation certifies
\(s_h\) as the unique oracle of \(\mathcal K_h\), conditional on saturation.

\begin{table}[!t]
\centering
\small
\setlength{\tabcolsep}{5pt}
\renewcommand{\arraystretch}{1.08}
\caption{\(\kappa_4\): logged-oracle selection and interval certification for
\(q=0.9\). Match uses \(E_{\rm quad}\) only for external validation; certified
counts are conditional on saturation.}
\label{tab:kappa4_interval_separation}
\begin{adjustbox}{max width=\textwidth}
\begin{tabular}{lccccc}
\toprule
Auxiliary pair & \(h\) & \(|\mathcal K_h|\) & Match & Certified
& \(\Gamma_{{\rm log},h}\) \\
\midrule
\multirow{4}{*}{\(V_{24}\subset V_{48}\)}
& 50  & 41 & \(5/5\) & \(0/5\) & \(1.079\pm0.048\) \\
& 100 & 21 & \(5/5\) & \(2/5\) & \(1.080\pm0.048\) \\
& 200 & 11 & \(4/5\) & \(4/5\) & \(1.088\pm0.065\) \\
& 400 & 6  & \(5/5\) & \(4/5\) & \(1.094\pm0.066\) \\
\addlinespace[0.35em]
\multirow{4}{*}{\(V_{48}\subset V_{96}\)}
& 50  & 41 & \(5/5\) & \(2/5\) & \(1.021\pm0.014\) \\
& 100 & 21 & \(5/5\) & \(5/5\) & \(1.021\pm0.014\) \\
& 200 & 11 & \(4/5\) & \(4/5\) & \(1.025\pm0.023\) \\
& 400 & 6  & \(5/5\) & \(5/5\) & \(1.027\pm0.023\) \\
\bottomrule
\end{tabular}
\end{adjustbox}
\end{table}

Auxiliary enrichment improves certifiability
(Table~\ref{tab:kappa4_interval_separation}): with
\(V_{48}\subset V_{96}\), all five \(\mathcal K_{100}\) selections are
certified and match the observed \(E_{\rm quad}\)-oracle. Higher certification
counts are observed on several coarser subarchives, although the dependence on
\(h\) is not monotone. On \(\mathcal K_{50}\), all five selections still match
but only two are certified, illustrating that failure of interval separation
is inconclusive rather than evidence of incorrect selection. The single
mismatch observed on \(\mathcal K_{200}\) is not certified and has only
\(0.72\%\) deterioration. No certified mismatch occurs.

\paragraph{Trajectory-wide archive control}
Certification within \(\mathcal K_h\) does not control checkpoints omitted by
logging. We therefore take \(\mathcal K_{50}\) as the prescribed finite
comparison trajectory and apply Section~\ref{sec:trajectory_wide_control} to
\(h\in\{100,200,400\}\) using \(V_{48}\subset V_{96}\) and \(q=0.9\).
For external validation we report
\[
\delta^{\rm traj}_{s_h}
:=100\left(
\frac{E_{{\rm quad},s_h}}{\min_{j\in\mathcal K_{50}}E_{{\rm quad},j}}-1
\right).
\]
The hatted quantities below are deterministic-quadrature realizations; they are
reference-free diagnostics relative to \(\mathcal K_{50}\), not certified
continuous-energy-norm bounds unless quadrature error is controlled.

\begin{table}[!t]
\centering
\scriptsize
\setlength{\tabcolsep}{3.0pt}
\renewcommand{\arraystretch}{1.10}
\caption{Trajectory-wide \(\kappa_4\) diagnostics relative to
\(\mathcal K_{50}\). Oracle inclusion is reported as observed/certified; in
these runs the selected-trajectory-oracle counts are identical. The last
column gives mean/max observed deterioration.}
\label{tab:kappa4_trajectory_wide}
\begin{adjustbox}{max width=\textwidth}
\begin{tabular}{ccccccc}
\toprule
\(h\) & \(|\mathcal K_h|\) & Oracle inclusion
& \(\widehat r_h\) & \(\widehat a_h+\widehat r_h\)
& \(\widehat\Gamma_{{\rm traj},h}\)
& \(\operatorname{mean/max}\,\delta^{\rm traj}_{s_h}\) \\
\midrule
100 & 21 & \(3/5\,/\,2/5\)
& \(0.012\pm0.013\) & \(0.015\pm0.015\)
& \(1.127\pm0.111\) & \(0.22\%/0.91\%\) \\
200 & 11 & \(1/5\,/\,1/5\)
& \(0.032\pm0.026\) & \(0.036\pm0.029\)
& \(1.397\pm0.334\) & \(3.38\%/7.91\%\) \\
400 & 6 & \(0/5\,/\,0/5\)
& \(0.061\pm0.041\) & \(0.065\pm0.045\)
& \(2.197\pm1.148\) & \(8.15\%/23.89\%\) \\
\bottomrule
\end{tabular}
\end{adjustbox}
\end{table}

The trajectory-wide results (Table~\ref{tab:kappa4_trajectory_wide}) separate
within-archive certifiability from trajectory coverage.
Although \(\mathcal K_{400}\) admits a logged-oracle certificate in all five
runs with \(V_{48}\subset V_{96}\), it contains no
\(\mathcal K_{50}\)-oracle. By contrast, \(\mathcal K_{100}\) halves the
archive from \(41\) to \(21\) checkpoints and limits the maximum observed
trajectory deterioration to \(0.91\%\). Coarser logging increases the reported trajectory-wide bounds and the observed
deterioration. All trajectory-wide certificates obtained here satisfy the external check \(q_{\rm obs}\le0.509<0.9\).
Thus auxiliary refinement controls selection within a fixed archive, whereas
logging refinement controls coverage of the prescribed trajectory.

We next test the same construction in the vector-valued energy geometry of
linear elasticity and under localized material contrast.

A separate post-training quadrature-ordering audit on the full
\(\kappa_4\) comparison archive (\(41\) checkpoints in each of five runs)
reevaluates \(\eta_{24}\), \(\eta_{48}\), \(\eta_{96}\), and a
quadrature-based Ritz score at Gauss orders \(Q=3,5,7\).
All four criteria select the \(Q=7\) reference-energy oracle in every run
at every tested order. The observed interquadrature score changes are much
smaller than the runner-up gaps (Table~\ref{tab:kappa4_quadrature_ordering}).
This provides an empirical illustration of ordering stability under the
tested numerical evaluations, without providing the certified perturbation
bound required by Corollary~\ref{cor:numerical_monitor_stability}.

\subsubsection{Plane-strain elasticity and smoothed high-contrast inclusion}
\label{subsubsec:elasticity_inclusion_results}

We consider plane-strain elasticity on \(\Omega=(0,1)^2\), with homogeneous
Dirichlet conditions, \(V_0=H_0^1(\Omega;\mathbb R^2)\), and
\[
a(u,v)=\int_\Omega
2\mu(x)\varepsilon(u):\varepsilon(v)
+\lambda(x)\operatorname{div}u\,\operatorname{div}v\,dx,
\]
where
\[
\varepsilon(v)=\tfrac12(\nabla v+\nabla v^\top),
\qquad
\sigma(v)=2\mu\varepsilon(v)+\lambda\operatorname{div}(v)I.
\]
The manufactured displacement is
\[
u^\ast(x,y)=
\begin{pmatrix}
\sin(\pi x)\sin(\pi y)\\[1mm]
\tfrac12\sin(2\pi x)\sin(\pi y)
\end{pmatrix},
\qquad
-\nabla\cdot\sigma(u^\ast)=f.
\]
With \(E_0=1\), \(\nu=0.30\), and the corresponding baseline Lam\'e
parameters \((\lambda_0,\mu_0)\), we set
\(\lambda(x)=s(x)\lambda_0\), \(\mu(x)=s(x)\mu_0\) and consider
\[
s_1=1,
\qquad
s_2=1+0.45\sin(6\pi x)\sin(4\pi y),
\]
plus the smoothed high-contrast inclusion
\[
s_{\rm inc}(x,y)=1+19\chi_\varepsilon(x,y),
\qquad
\chi_\varepsilon(x,y)
=
\frac12\left[
1-\tanh\!\left(
\frac{\sqrt{(x-0.55)^2+(y-0.52)^2}-0.18}{0.035}
\right)
\right].
\]
The three cases are denoted by \(\Ccal_1\), \(\Ccal_2\), and
\(\Ccal_{\rm inc}\).

\begin{figure}[!t]
\centering
\begin{subfigure}[t]{0.49\textwidth}
\centering
\includegraphics[width=\linewidth]{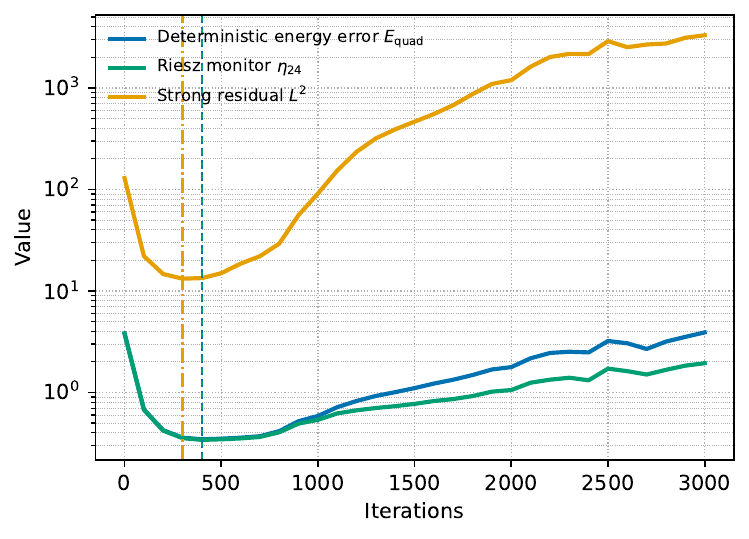}
\caption{Training diagnostics.}
\label{fig:elasticity_inclusion_monitor:a}
\end{subfigure}\hfill
\begin{subfigure}[t]{0.49\textwidth}
\centering
\includegraphics[width=\linewidth]{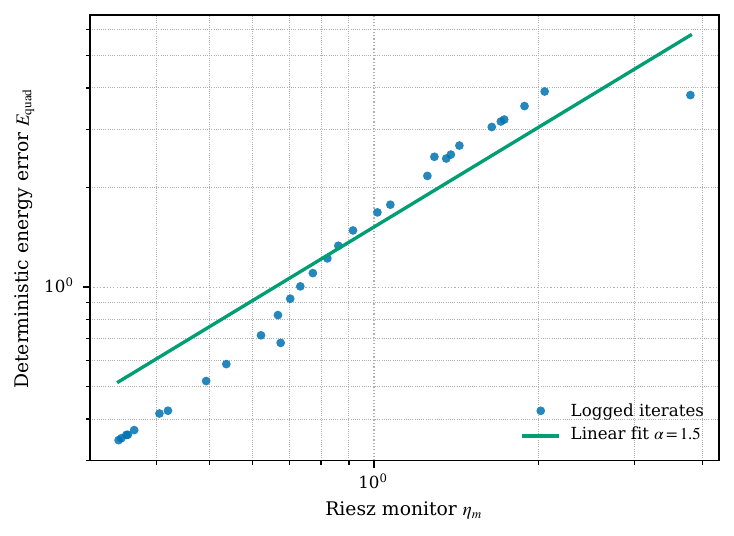}
\caption{Calibration against \(E_{\rm quad}\).}
\label{fig:elasticity_inclusion_monitor:b}
\end{subfigure}
\caption{Representative high-contrast inclusion run: temporal tracking and
calibration of the vector-valued monitor \(\eta_{24}\).}
\label{fig:elasticity_inclusion_monitor}
\end{figure}

Figure~\ref{fig:elasticity_inclusion_monitor} shows that the vector-valued
monitor remains aligned with the elastic energy error, including in the
high-contrast inclusion case.

\begin{table}[!t]
\centering
\footnotesize
\setlength{\tabcolsep}{4.5pt}
\renewcommand{\arraystretch}{1.12}
\caption{Elasticity: checkpoint selection over five seeds.}
\label{tab:elasticity_checkpoint_selection}
\begin{adjustbox}{max width=\textwidth}
\begin{tabular}{lcccccc}
\toprule
\multirow{2}{*}{Selection criterion}
& \multicolumn{2}{c}{\(\Ccal_1\)}
& \multicolumn{2}{c}{\(\Ccal_2\)}
& \multicolumn{2}{c}{\(\Ccal_{\rm inc}\)} \\
\cmidrule(lr){2-3}
\cmidrule(lr){4-5}
\cmidrule(lr){6-7}
& Step & \(D_c\) & Step & \(D_c\) & Step & \(D_c\) \\
\midrule
Energy oracle \(\min E_{\rm quad}\)
& \(960 \pm 114\) & \(1.000 \pm 0.000\)
& \(960 \pm 55\)  & \(1.000 \pm 0.000\)
& \(400 \pm 122\) & \(1.000 \pm 0.000\) \\
Riesz monitor \(\min \eta_{24}\)
& \(960 \pm 114\) & \(1.000 \pm 0.000\)
& \(960 \pm 55\)  & \(1.000 \pm 0.000\)
& \(420 \pm 130\) & \(1.000 \pm 0.000\) \\
Strong residual \(L^2\)
& \(680 \pm 84\)  & \(1.249 \pm 0.141\)
& \(700 \pm 71\)  & \(1.206 \pm 0.127\)
& \(320 \pm 110\) & \(1.011 \pm 0.016\) \\
Training Ritz loss \(\min J_{\rm train}\)
& \(3000 \pm 0\) & \(6.597 \pm 2.770\)
& \(3000 \pm 0\) & \(6.819 \pm 2.047\)
& \(3000 \pm 0\) & \(12.287 \pm 3.022\) \\
Validation Ritz loss \(\min J_{\rm val}\)
& \(940 \pm 89\)  & \(1.011 \pm 0.024\)
& \(940 \pm 55\)  & \(1.000 \pm 0.001\)
& \(500 \pm 245\) & \(1.059 \pm 0.089\) \\
\bottomrule
\end{tabular}
\end{adjustbox}
\end{table}

The Riesz criterion is oracle-level in all three material cases
(Table~\ref{tab:elasticity_checkpoint_selection}); in one
inclusion seed it selects a neighboring checkpoint with negligible
deterioration \((D_c=1.000077)\). The training Ritz loss always selects the
final logged iterate with mean deterioration factors \(D_c=6.60\)--\(12.29\), while independent validation is substantially closer to
the oracle. As in the scalar benchmarks, the comparison with \(J_{\rm val}\) is empirical,
since finite-set evaluations may perturb the ordering of the continuous Ritz
functional.

\begin{table}[!t]
\centering
\small
\setlength{\tabcolsep}{6pt}
\renewcommand{\arraystretch}{1.12}
\caption{Elasticity: calibration and hierarchical diagnostics at the
\(\eta_{24}\)-selected checkpoint.}
\label{tab:elasticity_summary}
\begin{adjustbox}{max width=\textwidth}
\begin{tabular}{lcccc}
\toprule
\multirow{2}{*}{Material}
& \multicolumn{2}{c}{Monitor calibration}
& \multicolumn{2}{c}{Hierarchical estimator} \\
\cmidrule(lr){2-3}
\cmidrule(lr){4-5}
& \(\eta_{24}/E_{\rm quad}\)
& \(\eta_{48}/E_{\rm quad}\)
& \(q_{\rm obs}\)
& \(\Ucal_{24,48,0.9}/E_{\rm quad}\) \\
\midrule
\(\Ccal_1\)
& \(0.971 \pm 0.007\)
& \(0.993 \pm 0.002\)
& \(0.507 \pm 0.002\)
& \(1.079 \pm 0.019\) \\
\(\Ccal_2\)
& \(0.973 \pm 0.008\)
& \(0.993 \pm 0.002\)
& \(0.507 \pm 0.002\)
& \(1.074 \pm 0.020\) \\
\(\Ccal_{\rm inc}\)
& \(0.986 \pm 0.004\)
& \(0.996 \pm 0.001\)
& \(0.509 \pm 0.002\)
& \(1.040 \pm 0.011\) \\
\bottomrule
\end{tabular}
\end{adjustbox}
\end{table}

At the operational level (Table~\ref{tab:elasticity_summary}), the mean
calibration ratio \(\eta_{24}/E_{\rm quad}\) ranges from \(0.971\) to
\(0.986\); level \(48\) raises the mean ratio to \(0.993\)--\(0.996\). The external
saturation check again gives \(q_{\rm obs}\simeq0.5\), while the mean
\(q=0.9\) upper-estimate ratio ranges from \(1.040\) to \(1.079\).

Across the three elasticity cases, \(\eta_m/E_{\rm quad}\) again increases
toward unity under enrichment
(Figure~\ref{fig:manufactured_aux_refinement:b}), in agreement with the
refinement behavior established in
Corollary~\ref{cor:conforming_refinement}.

Thus oracle-level selection and tight energy-scale calibration persist in
vector-valued elasticity with localized material contrast. A complementary
spatial validation in Figure~\ref{fig:app_elasticity_inclusion_spatial} shows
that the Riesz-projected density captures the dominant energetic concentration
associated with the localized inclusion. We next stress the auxiliary-resolution
requirement using a reentrant-corner singularity.

\begin{figure}[!t]
\centering
\begin{subfigure}[t]{0.32\textwidth}
\centering
\includegraphics[width=\linewidth]{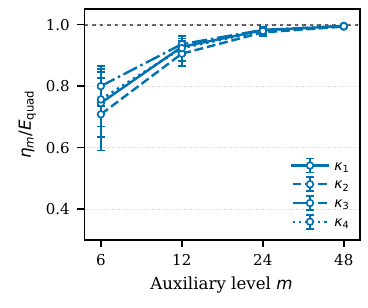}
\caption{Scalar diffusion.}
\label{fig:manufactured_aux_refinement:a}
\end{subfigure}\hfill
\begin{subfigure}[t]{0.32\textwidth}
\centering
\includegraphics[width=\linewidth]{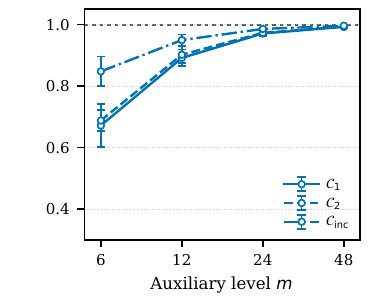}
\caption{Plane-strain elasticity.}
\label{fig:manufactured_aux_refinement:b}
\end{subfigure}\hfill
\begin{subfigure}[t]{0.32\textwidth}
\centering
\includegraphics[width=\linewidth]{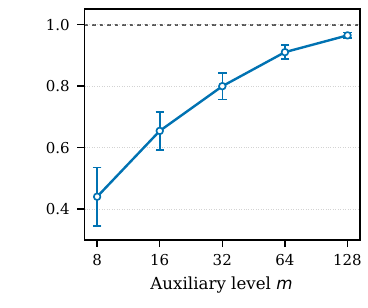}
\caption{L-shaped domain.}
\label{fig:manufactured_aux_refinement:c}
\end{subfigure}
\caption{Auxiliary refinement at the operationally selected checkpoints for
the manufactured benchmarks. Curves report the mean
\(\eta_m/E_{\rm quad}\) over five seeds, with error bars denoting one standard
deviation. Within each run, the neural checkpoint is held fixed as \(m\)
varies; the dashed line marks unit calibration.}
\label{fig:manufactured_aux_refinement}
\end{figure}

\subsubsection{L-shaped-domain stress test}
\label{subsubsec:lshape_results}

The final manufactured benchmark uses
\[
\Omega_L=(-1,1)^2\setminus\big([0,1)\times(-1,0]\big),
\qquad
-\Delta u=f,
\qquad
u=0\ \text{on }\partial\Omega_L,
\]
with
\[
u^\ast(x,y)
=(1-x^2)(1-y^2)r^{2/3}\sin\!\left(\frac{2\theta}{3}\right),
\qquad
f=-\Delta u^\ast.
\]
Here \((r,\theta)\) are polar coordinates centered at the reentrant corner,
with \(\theta\in[0,3\pi/2]\).
The reentrant-corner singularity tests checkpoint selection when uniform
auxiliary spaces require substantially greater resolution.

\begin{figure}[!t]
\centering
\begin{subfigure}[t]{0.49\textwidth}
\centering
\includegraphics[width=\linewidth]
{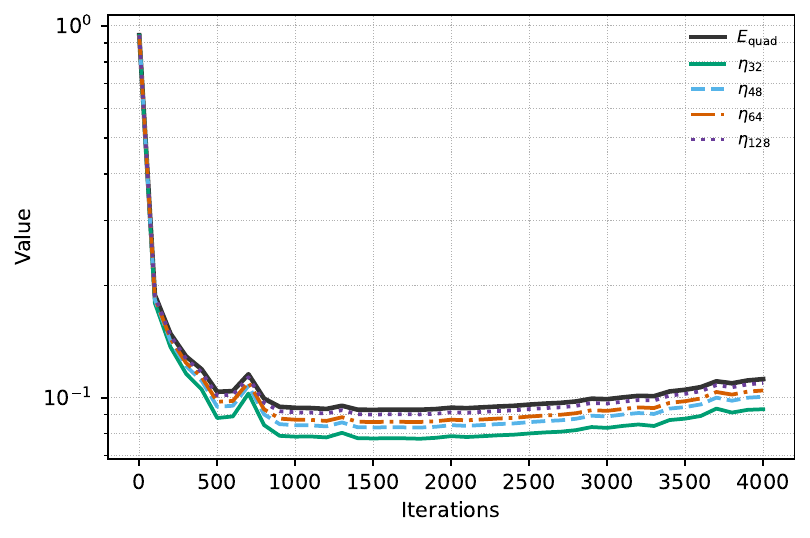}
\caption{Training diagnostics.}
\label{fig:lshape_monitor:a}
\end{subfigure}\hfill
\begin{subfigure}[t]{0.49\textwidth}
\centering
\includegraphics[width=\linewidth]
{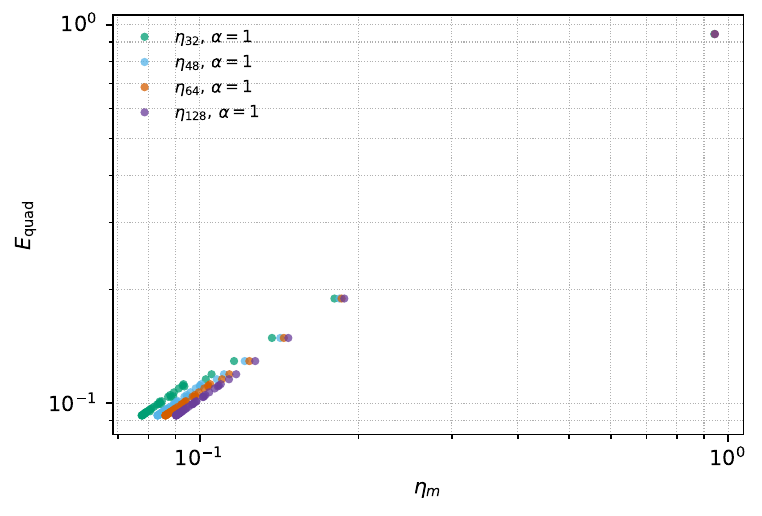}
\caption{Calibration against \(E_{\rm quad}\).}
\label{fig:lshape_monitor:b}
\end{subfigure}
\caption{L-shaped-domain monitor comparison for a representative run.
The conforming Riesz monitors at increasing auxiliary resolutions track the
reference energy error along the logged trajectory and approach the
energy-error scale under enrichment.}
\label{fig:lshape_monitor}
\end{figure}

Figure~\ref{fig:lshape_monitor} shows that all monitored levels track the
evolution of \(E_{\rm quad}\), while auxiliary enrichment improves their
energy-scale calibration. The stronger resolution requirement is consistent
with the reentrant-corner singularity.

\begin{table}[!t]
\centering
\small
\setlength{\tabcolsep}{6pt}
\renewcommand{\arraystretch}{1.10}
\caption{L-shaped domain: checkpoint selection over five seeds. The strong
residual uses the cutoff \(r>0.035\) around the reentrant corner.}
\label{tab:lshape_checkpoint_selection}
\begin{tabular}{lcc}
\toprule
Selection criterion & Selected step & \(D_c\) \\
\midrule
\(\min \eta_{32}\)
& \(1900 \pm 1086\) & \(1.003 \pm 0.005\) \\
\(\min \eta_{48}\)
& \(1980 \pm 1031\) & \(1.000 \pm 0.001\) \\
\(\min \eta_{64}\)
& \(1960 \pm 1031\) & \(1.000 \pm 0.000\) \\
\(\min \eta_{128}\)
& \(1960 \pm 1031\) & \(1.000 \pm 0.000\) \\
\addlinespace[0.3em]
Strong residual \(L^2(r>0.035)\)
& \(1640 \pm 767\) & \(1.007 \pm 0.006\) \\
Training Ritz loss \(\min J_{\rm train}\)
& \(4000 \pm 0\) & \(1.181 \pm 0.095\) \\
Validation Ritz loss \(\min J_{\rm val}\)
& \(1800 \pm 1251\) & \(1.092 \pm 0.169\) \\
\bottomrule
\end{tabular}
\end{table}

The light monitor \(\eta_{32}\) is already near-oracle, while \(\eta_{64}\)
and \(\eta_{128}\) recover the logged oracle in all five runs
(Table~\ref{tab:lshape_checkpoint_selection}). Exact logged-oracle recovery across all five runs is therefore reached only at
a finer auxiliary resolution in this benchmark than in the smooth
square-domain tests, consistent with
Corollary~\ref{cor:oracle_recovery}.

\begin{table}[!t]
\centering
\small
\setlength{\tabcolsep}{5pt}
\renewcommand{\arraystretch}{1.08}
\caption{L-shaped domain: selected auxiliary-resolution diagnostics.}
\label{tab:lshape_summary}
\begin{tabular}{lcccc}
\toprule
Level/diagnostic
& \(\eta_m/E_{\rm quad}\)
& \(\rho_{m,{\rm rel}}^{\rm quad}\)
& \(\gamma\) or \(q_{\rm obs}\)
& \(\Ucal/E_{\rm quad}\) \\
\midrule
\(32\times32\) & \(0.801\pm0.044\) & \(0.595\pm0.061\) & -- & -- \\
\(64\times64\) & \(0.911\pm0.022\) & \(0.409\pm0.050\) & -- & -- \\
\(128\times128\) & \(0.965\pm0.009\) & \(0.260\pm0.034\) & -- & -- \\
\midrule
\((32,64,128)\) & -- & -- & \(0.729\pm0.037\) & -- \\
\((64,128)\) & -- & -- & \(0.637\pm0.007\) & \(1.165\pm0.035\) \\
\bottomrule
\end{tabular}
\end{table}

At the \(\eta_{64}\)-selected checkpoint (Table~\ref{tab:lshape_summary}),
the captured fraction rises from
\(0.801\pm0.044\) at level \(32\) to \(0.965\pm0.009\) at level \(128\),
while the unresolved fraction decreases accordingly. For
\(V_{64}\subset V_{128}\), the external check gives
\(q_{\rm obs}=0.637\pm0.007<0.9\) and the conditional upper estimate is
\(1.165\pm0.035\) times \(E_{\rm quad}\). For the reported L-shaped runs, we
use \(\eta_{64}\) as the operational selector and \(\eta_{128}\) as the
enriched post-training check. Refinement remains monotone and progressively
improves energy-scale calibration
(Figure~\ref{fig:manufactured_aux_refinement:c}), consistently with
Corollary~\ref{cor:conforming_refinement}.

Notably, \(\eta_{64}/E_{\rm quad}=0.911\pm0.022\) already recovers the logged
oracle in all five runs. Thus, in this benchmark, oracle recovery precedes
near-unit calibration: the archive ordering is already resolved sufficiently
for oracle recovery before near-exact recovery of the error scale is reached,
illustrating the finite-gap mechanism of
Corollary~\ref{cor:oracle_recovery}. A complementary spatial validation in
Figure~\ref{fig:app_lshape_spatial} shows that the enriched Riesz
reconstruction captures the dominant energy concentration around the
reentrant corner.

We next quantify the cost of the operational auxiliary levels.

\subsubsection{Operational cost of checkpoint selection}
\label{subsubsec:cost_operational}

\begin{table}[!t]
\centering
\small
\setlength{\tabcolsep}{4pt}
\caption{Representative cost of the operational and enriched Riesz
reconstructions. Overhead is measured against the pure optimizer-update
time between logged checkpoints.}
\label{tab:computational_cost_summary}
\begin{adjustbox}{max width=\textwidth}
\begin{tabular}{lccccc}
\toprule
Family
& \(M_{\rm op}\)
& Operational call (s)
& Operational overhead
& Enriched level
& Enriched overhead \\
\midrule
Scalar diffusion
& \(24\) & \(0.034\)--\(0.045\) & \(0.58\)--\(0.76\%\)
& \(48\) & \(12.32\)--\(15.30\%\) \\
Elasticity/inclusion
& \(24\) & \(1.881\)--\(2.045\) & \(8.01\)--\(8.72\%\)
& \(48\) & \(8.27\)--\(9.01\%\) \\
L-shaped domain
& \(64\) & \(0.131\pm0.009\) & \(0.74\pm0.03\%\)
& \(128\) & \(4.04\pm0.74\%\) \\
\bottomrule
\end{tabular}
\end{adjustbox}
\end{table}

The operational monitor incurs less than \(1\%\) overhead in the scalar and
L-shaped tests and \(8.01\%\)--\(8.72\%\) in elasticity
(Table~\ref{tab:computational_cost_summary}). These costs support its practical
use for checkpoint ranking along the logged trajectory, while enriched
reconstructions are reserved for tighter post-training assessment and
certification.

We finally test the selector without a closed-form solution, using FEM only
for external post-training validation.

\subsection{Perforated plate: reference-free selection with external FEM validation}
\label{subsec:perforated_plate_results}

We consider plane-strain elasticity on
\[
\Omega=(0,1)^2\setminus\overline{P_{96}(c,r)},
\qquad
c=(1/2,1/2),
\qquad
r=0.16,
\]
where \(P_{96}(c,r)\) is a fixed polygonal approximation of a circular hole.
The material is homogeneous with \(E=1\), \(\nu=0.30\), zero body force, and
outer-boundary displacement
\[
u_D(x,y)=\bigl(\varepsilon_0x+\gamma_0y,-\nu\varepsilon_0y\bigr),
\qquad
\varepsilon_0=5\times10^{-2},
\qquad
\gamma_0=2.5\times10^{-2},
\]
while the hole boundary is traction-free. The neural approximation satisfies
the prescribed outer-boundary displacement strongly, and we use the nested
auxiliary hierarchy \(V_{24}\subset V_{48}\subset V_{96}\subset V_0\)
described above.

No closed-form solution is available. Both external FEM references are
conforming vector-valued \(P_1\) Galerkin solutions on fitted triangular
meshes of the same fixed polygonal domain. The level-\(256\) reference,
used only for external post-training validation, has \(179\,882\) displacement
degrees of freedom, while the level-\(192\) check has \(109\,608\). The latter
yields the same FEM-reference oracle in all five runs and changes checkpoint
errors by at most \(6.3\times10^{-4}\) relatively. Neither FEM solve enters
training, Riesz reconstruction, or checkpoint selection.

\begin{figure}[!t]
\centering
\begin{subfigure}[t]{0.55\textwidth}
\centering
\includegraphics[width=\linewidth]{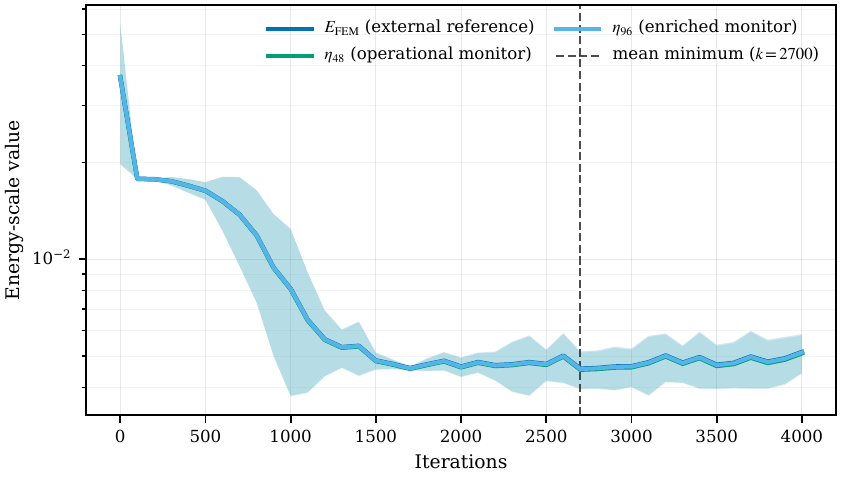}
\caption{Training diagnostics.}
\label{fig:plate_hole_fem_reference:a}
\end{subfigure}
\hfill
\begin{subfigure}[t]{0.37\textwidth}
\centering
\includegraphics[width=\linewidth]{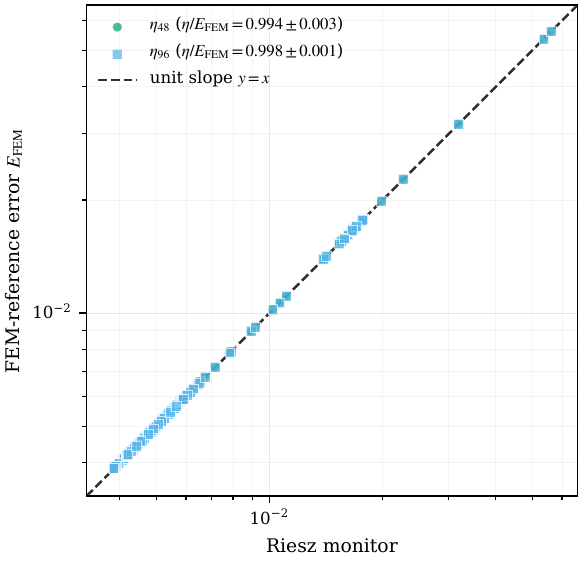}
\caption{Calibration against the FEM reference.}
\label{fig:plate_hole_fem_reference:b}
\end{subfigure}
\caption{Perforated plate: external FEM validation of the operational
\(\eta_{48}\) monitor over five runs.}
\label{fig:plate_hole_fem_reference}
\end{figure}

Figure~\ref{fig:plate_hole_fem_reference} shows that both \(\eta_{48}\) and
\(\eta_{96}\) remain closely aligned with \(E_{\rm FEM}\) along the logged
trajectory.

\begin{table}[!t]
\centering
\small
\setlength{\tabcolsep}{6pt}
\caption{Perforated plate: checkpoint selection against the external FEM
reference.}
\label{tab:plate_hole_fem_reference}
\begin{tabular}{lcc}
\toprule
Criterion \(c\) & Step & \(D_c^{\rm FEM}\) \\
\midrule
final & \(4000\pm0\) & \(1.204\pm0.138\) \\
\(\min E_{\rm FEM}\) & \(2640\pm873\) & \(1.000\pm0.000\) \\
\(\min \eta_{48}\) & \(2640\pm873\) & \(1.000\pm0.000\) \\
\(\min \eta_{96}\) & \(2640\pm873\) & \(1.000\pm0.000\) \\
\(\min J_{\rm train}\) & \(3900\pm100\) & \(1.139\pm0.161\) \\
\(\min J_{\rm val}\) & \(2720\pm1064\) & \(1.018\pm0.024\) \\
\(\min R_{\rm strong}\) & \(360\pm152\) & \(4.146\pm0.274\) \\
\bottomrule
\end{tabular}
\end{table}

Both \(\eta_{48}\) and \(\eta_{96}\) select the FEM-reference oracle in all
five runs (Table~\ref{tab:plate_hole_fem_reference}). The final iterate is
\(20.4\%\) worse on average, validation Ritz
selection remains close to the oracle, and the strong residual selects a much
less accurate checkpoint.

We next assess the sharpness of the reference-free hierarchy at the selected
checkpoint.

\begin{table}[!t]
\centering
\small
\setlength{\tabcolsep}{7pt}
\caption{Perforated plate: hierarchical diagnostics at the
\(\eta_{48}\)-selected checkpoint. Ratios involving \(E_{\rm FEM}\) are
external validation quantities; \(\Ucal_{48,96,0.9}/\eta_{48}\) is the
reference-free conditional bracket factor.}
\label{tab:plate_hole_hierarchical_calibration}
\begin{tabular}{lc}
\toprule
Quantity & Value \\
\midrule
\(\eta_{24}/E_{\rm FEM}\) & \(0.977\pm0.002\) \\
\(\eta_{48}/E_{\rm FEM}\) & \(0.993\pm0.001\) \\
\(\eta_{96}/E_{\rm FEM}\) & \(0.998\pm0.000\) \\
\midrule
\(\delta_{48,96}/E_{\rm FEM}\) & \(0.093\pm0.005\) \\
\(\gamma_{24,48,96}\) & \(0.517\pm0.001\) \\
\midrule
\(\Ucal_{48,96,0.9}/\eta_{48}\) & \(1.023\pm0.002\) \\
\(\Ucal_{48,96,0.9}/E_{\rm FEM}\) & \(1.016\pm0.002\) \\
\bottomrule
\end{tabular}
\end{table}

The auxiliary hierarchy is tightly calibrated on the FEM-reference scale
(Table~\ref{tab:plate_hole_hierarchical_calibration}):
\(\eta_m/E_{\rm FEM}\) increases from \(0.977\pm0.002\) at level \(24\) to
\(0.998\pm0.000\) at level \(96\). Independently of the FEM reference, the
conditional \(q=0.9\) bracket has upper-to-lower factor
\(1.023\pm0.002\); the FEM comparison gives the external check
\(\Ucal_{48,96,0.9}/E_{\rm FEM}=1.016\pm0.002\).

To assess the mechanical impact of checkpoint selection, we compare the
\(\eta_{48}\)-selected output with the final iterate at the solution level.

\begin{table}[!t]
\centering
\small
\setlength{\tabcolsep}{4pt}
\caption{Perforated plate: solution-level comparison with the refined FEM
reference. Here
\(e_{u,L^2}=\|u_\theta-u_{\rm FEM}^{\rm ref}\|_{L^2}/
\|u_{\rm FEM}^{\rm ref}\|_{L^2}\),
\(e_\varepsilon=\|\varepsilon(u_\theta-u_{\rm FEM}^{\rm ref})\|_{L^2}/
\|\varepsilon(u_{\rm FEM}^{\rm ref})\|_{L^2}\), and
\(e_{\rm vm}=\|\sigma_{\rm vm}(u_\theta)-\sigma_{\rm vm}(u_{\rm FEM}^{\rm ref})\|_{L^2}/
\|\sigma_{\rm vm}(u_{\rm FEM}^{\rm ref})\|_{L^2}\),
where \(\sigma_{\rm vm}\) denotes the plane-strain von Mises equivalent stress.
The sampled maximum-displacement error is
\(e_{u,\max}:=\max_{x\in\mathcal S}
\|u_\theta(x)-u_{\rm FEM}^{\rm ref}(x)\|_2\),
where \(\mathcal S\) is the union of the reference-mesh nodes and the
degree-two triangle quadrature points; the final row reports final/selected
ratios.}
\label{tab:plate_hole_solution_level}
\begin{adjustbox}{max width=\textwidth}
\begin{tabular}{lccccc}
\toprule
Selected output
& \(E_{\rm FEM}\)
& \(e_{u,L^2}\)
& \(e_\varepsilon\)
& \(e_{\rm vm}\)
& sampled \(e_{u,\max}\) \\
\midrule
\(\min \eta_{48}\)
& \(4.27{\times}10^{-3}\pm2.65{\times}10^{-4}\)
& \(0.006\pm0.001\)
& \(0.085\pm0.005\)
& \(0.061\pm0.002\)
& \(1.22{\times}10^{-3}\pm1.50{\times}10^{-4}\) \\
final
& \(5.15{\times}10^{-3}\pm7.10{\times}10^{-4}\)
& \(0.008\pm0.002\)
& \(0.097\pm0.011\)
& \(0.073\pm0.012\)
& \(1.56{\times}10^{-3}\pm2.30{\times}10^{-4}\) \\
\midrule
final / \(\min \eta_{48}\)
& \(1.204\) & \(1.405\) & \(1.140\) & \(1.206\) & \(1.275\) \\
\bottomrule
\end{tabular}
\end{adjustbox}
\end{table}

The late-iterate degradation is also visible in the solution-level errors
(Table~\ref{tab:plate_hole_solution_level}): relative to the
\(\eta_{48}\)-selected output, the final iterate
increases the displacement, strain, von Mises stress, and sampled maximum-
displacement errors by factors \(1.405\), \(1.140\), \(1.206\), and
\(1.275\), respectively.

\begin{table}[!t]
\centering
\small
\setlength{\tabcolsep}{5pt}
\caption{Perforated plate: external FEM-solve cost and Riesz monitoring cost.
FEM times are one-time reference solves used only for validation; Riesz times
are mean costs per logged checkpoint, with overhead measured against the pure
optimizer-update time between logs.}
\label{tab:plate_hole_cost}
\begin{tabular}{lccc}
\toprule
Quantity & DOFs & Time (s) & Overhead \\
\midrule
FEM ref. \(256\) & 179882 & \(53.39\) & -- \\
FEM check \(192\) & 109608 & \(34.47\) & -- \\
\midrule
\(\eta_{24}\) & 2642 & \(0.212\pm0.055\) & \(0.93\%\) \\
\(\eta_{48}\) & 10536 & \(0.959\pm0.250\) & \(4.19\%\) \\
\(\eta_{96}\) & 42080 & \(3.621\pm0.677\) & \(15.84\%\) \\
\bottomrule
\end{tabular}
\end{table}

The operational \(\eta_{48}\) monitor adds \(4.19\%\) overhead in this
benchmark (Table~\ref{tab:plate_hole_cost}); the more expensive \(\eta_{96}\) reconstruction is used only for
post-training qualification. Thus the reference-free selector operates at moderate cost without requiring
the refined FEM solve in the selection pipeline.

Taken together, the manufactured benchmarks validate energy-scale calibration,
refinement, and oracle-level selection, while the perforated-plate experiment
supports the use of the Riesz monitor as a lightweight, reference-free
post-training selector when no exact solution is available.

\FloatBarrier

\section{Discussion and outlook}
\label{sec:discussion}

This work developed a reference-free archive-level checkpoint-selection
framework for admissible neural approximations of symmetric coercive
variational problems. Conforming Riesz reconstruction converts the exact
residual--energy geometry into a computable, training-independent monitor,
making the logged energy oracle recoverable without the exact solution or a
reference solve. Monitor refinement yields eventual logged-oracle recovery,
while, under saturation, nested reconstructions provide finite-resolution
near-oracle bounds and interval-separation certificates. Logging-resolution
control further quantifies the loss relative to a prescribed finite comparison
trajectory and provides trajectory-oracle certificates. Across the numerical
benchmarks, the monitor approaches the reference energy error under refinement
and recovers oracle-level checkpoints once sufficiently resolved. On the
perforated plate, this behavior is confirmed against an independent
FEM reference, with modest observed post-processing cost.

A central conceptual contribution of this work is to
formulate reliable neural checkpoint selection as an
archive-level ordering problem, distinct from
checkpointwise error assessment. The key mechanism
is preservation of the oracle--non-oracle separation.
At finite auxiliary resolution, conforming Riesz
monitors may provide valid lower bounds for individual
energy errors while selecting a non-oracle checkpoint,
as unresolved components vary across the archive.
We characterize this ordering obstruction and
establish that uniform recovery over the finite
archive eventually restores oracle selection.
This analysis provides a mathematical foundation
for assessing the reliability of the neural
approximation retained after training.

The choice of conforming Riesz reconstruction is motivated by
the joint requirements of error quantification and reliable archive
selection. Exact Ritz values preserve energy-error ordering but do
not quantify absolute errors without the unknown optimal energy.
Strong-form residual norms generally lack exact energy calibration,
while equilibrated reconstructions involve problem-dependent
constructions. For symmetric coercive problems, the energy-Riesz
representative coincides with the signed error, and its conforming
projection yields the exact decomposition
\[
E_k^2
=
\eta_{m,k}^2
+
\|z_k-z_{m,k}\|_a^2.
\]
This identity exposes how checkpoint-dependent unresolved errors
can reverse rankings and supports monotone recovery and conditional
oracle certification. Conforming Riesz reconstruction thus offers a unified,
energy-calibrated framework for error assessment and reliable
logged-oracle selection.

The ordering question extends beyond neural PDE
approximations, but the present symmetric coercive
setting provides the exact calibration
\[
\|R(u_{\theta_k})\|_{V_{0,a}'}=E_k.
\]
Consequently, the continuous residual and
energy-error orderings coincide, and the obstruction
analyzed here arises from finite-dimensional
reconstruction. Beyond this energy-isometric setting,
however, even exact residual scores may fail to
preserve the target-error ordering: residual--error
norm equivalence alone is insufficient. Let
\(X=Y=\mathbb R^2\) with the Euclidean norm and
\[
B=\operatorname{diag}(1,2),
\qquad
e_i=(0,1),
\qquad
e_j=(3/2,0).
\]
Then
\[
\|e\|_2\leq\|Be\|_2\leq2\|e\|_2
\qquad
\forall e\in\mathbb R^2,
\]
but
\[
\|e_i\|_2<\|e_j\|_2,
\qquad
\|Be_i\|_2>\|Be_j\|_2.
\]
This identifies two distinct sources of ordering
distortion in more general variational formulations:
a discrepancy between the exact residual score
and the target error, and the additional discrepancy
introduced by numerical score approximation.
Extending selection guarantees to such settings
calls for archive-dependent conditions that control
both effects.

The distinction between direct Ritz--Galerkin
approximation and the selection-oriented auxiliary
reconstruction discussed in
Subsection~\ref{subsec:conforming_riesz_hierarchy}
has an operational consequence.
 The auxiliary
resolution is assessed by its ability to identify
an energy-optimal checkpoint, rather than by a
prescribed absolute accuracy for a new primal
approximation. Corollary~\ref{cor:oracle_recovery}
provides the sufficient condition
$\varepsilon_m<\delta_E$ for logged-oracle recovery.
In the L-shaped-domain experiment, $\eta_{64}$
selects the logged oracle in all five runs despite
incomplete error-scale calibration at the selected
checkpoint ($\eta_{64}/E_{\rm quad}=0.911\pm0.022$).
At this operational resolution, monitoring incurs
only $0.74\pm0.03\%$ overhead relative to
optimizer-update time between logged checkpoints
(Tables~\ref{tab:lshape_checkpoint_selection}--%
\ref{tab:lshape_summary}
and~\ref{tab:computational_cost_summary}).
This illustrates that correct selection can precede
near-exact error calibration at modest observed
monitoring cost.

The ordering analysis also extends to numerically
evaluated monitors. Corollary~\ref{cor:numerical_monitor_stability}
bounds the selection regret by the uniform projection
deficit plus twice the numerical monitor-evaluation
error. This separates auxiliary-resolution effects
from numerical perturbations and gives a sufficient
condition for preserving the logged oracle.
The guarantee remains conditional on the availability
of a valid numerical error bound. In the full \(\kappa_4\) archive audit,
the evaluated Riesz and quadrature-based Ritz criteria select the same
checkpoints across \(Q=3,5,7\), with interquadrature variations small
relative to the observed selection gaps
(Table~\ref{tab:kappa4_quadrature_ordering}).
These observations support numerical ordering stability for the tested
quadratures, but neither provide a certified \(\tau_m\) nor establish a
general selection advantage over accurately quadrature-evaluated Ritz scores.

Further directions include saturation-free
finite-resolution certification and explicit
quadrature-error control to strengthen the
trajectory-wide guarantees in the continuous
energy norm. The framework concerns selection
among candidates generated by the optimizer. A further direction is to carry the
order-preservation principle upstream into
the optimization procedure and to explore
the potential of conforming Riesz reconstruction
for improving neural approximations themselves.

\appendix

\section{Sensitivity to the prescribed saturation parameter}
\label{app:q_sensitivity}

The prescribed saturation parameter \(q\) enters only the conditional upper
estimate
\(\Ucal_{m,M,q}
=(\eta_m^2+\delta_{m,M}^2/(1-q^2))^{1/2}\)
and therefore does not affect monitor-based checkpoint selection. We assess
the sensitivity of the finite-resolution certificates and upper estimates over
the range \(q\in\{0.70,0.80,0.90,0.96\}\).

Table~\ref{tab:app_kappa4_q_sensitivity} reports the resulting
interval-separation certification counts on \(\kappa_4\) over five independent
trajectories.

\begin{table}[!hbtp]
\centering
\caption{Sensitivity of interval-separation certification on \(\kappa_4\).
Each entry is the number of certified runs out of five.}
\label{tab:app_kappa4_q_sensitivity}
\scriptsize
\setlength{\tabcolsep}{3.2pt}
\begin{adjustbox}{max width=\textwidth}
\begin{tabular}{c@{\qquad}cccc@{\qquad}cccc}
\toprule
& \multicolumn{4}{c}{\(V_{24}\subset V_{48}\)}
& \multicolumn{4}{c}{\(V_{48}\subset V_{96}\)} \\
\cmidrule(lr){2-5}
\cmidrule(lr){6-9}
\(q\)
& \(\mathcal K_{50}\)
& \(\mathcal K_{100}\)
& \(\mathcal K_{200}\)
& \(\mathcal K_{400}\)
& \(\mathcal K_{50}\)
& \(\mathcal K_{100}\)
& \(\mathcal K_{200}\)
& \(\mathcal K_{400}\) \\
\midrule
\(0.70\)
& \(2/5\) & \(5/5\) & \(4/5\) & \(5/5\)
& \(3/5\) & \(5/5\) & \(4/5\) & \(5/5\) \\
\(0.80\)
& \(1/5\) & \(4/5\) & \(4/5\) & \(5/5\)
& \(3/5\) & \(5/5\) & \(4/5\) & \(5/5\) \\
\(0.90\)
& \(0/5\) & \(2/5\) & \(4/5\) & \(4/5\)
& \(2/5\) & \(5/5\) & \(4/5\) & \(5/5\) \\
\(0.96\)
& \(0/5\) & \(0/5\) & \(2/5\) & \(4/5\)
& \(0/5\) & \(3/5\) & \(4/5\) & \(5/5\) \\
\bottomrule
\end{tabular}
\end{adjustbox}
\end{table}

The certification pattern remains stable over a broad range of \(q\),
particularly for \(V_{48}\subset V_{96}\), with the expected gradual loss of
certificates as the upper bounds become more conservative. Larger counts on
coarser subarchives reflect greater checkpoint separation rather than improved
trajectory coverage.

For the perforated plate, Table~\ref{tab:app_plate_hole_q_sensitivity} reports
the conditional upper estimate relative to the independent FEM-reference error
at the \(\eta_{48}\)-selected checkpoint.

\begin{table}[!hbtp]
\centering
\caption{Sensitivity of the conditional upper estimate on the perforated
plate.}
\label{tab:app_plate_hole_q_sensitivity}
\footnotesize
\setlength{\tabcolsep}{5pt}
\begin{tabular}{cc}
\toprule
\(q\) & \(\Ucal_{48,96,q}/E_{\rm FEM}\) \\
\midrule
\(0.70\) & \(1.002\pm0.000\) \\
\(0.80\) & \(1.005\pm0.001\) \\
\(0.90\) & \(1.016\pm0.002\) \\
\(0.96\) & \(1.048\pm0.005\) \\
\bottomrule
\end{tabular}
\end{table}

The upper estimate remains close to the FEM-reference error throughout the
tested range, reaching only \(1.048\pm0.005\) at \(q=0.96\). Thus the qualitative conclusions remain stable over the tested range of \(q\),
with the expected increase in conservatism as \(q\) grows.

\section{Spatial diagnostics of the conforming Riesz reconstruction}
\label{app:spatial_diagnostics}

For two manufactured stress tests, we compare the reference energy-error
density with the spatial density of the conforming Riesz reconstruction.
These comparisons provide a qualitative view of the dominant energetic
regions captured by the reconstruction.

\begin{figure}[!t]
\centering
\includegraphics[width=\textwidth]
{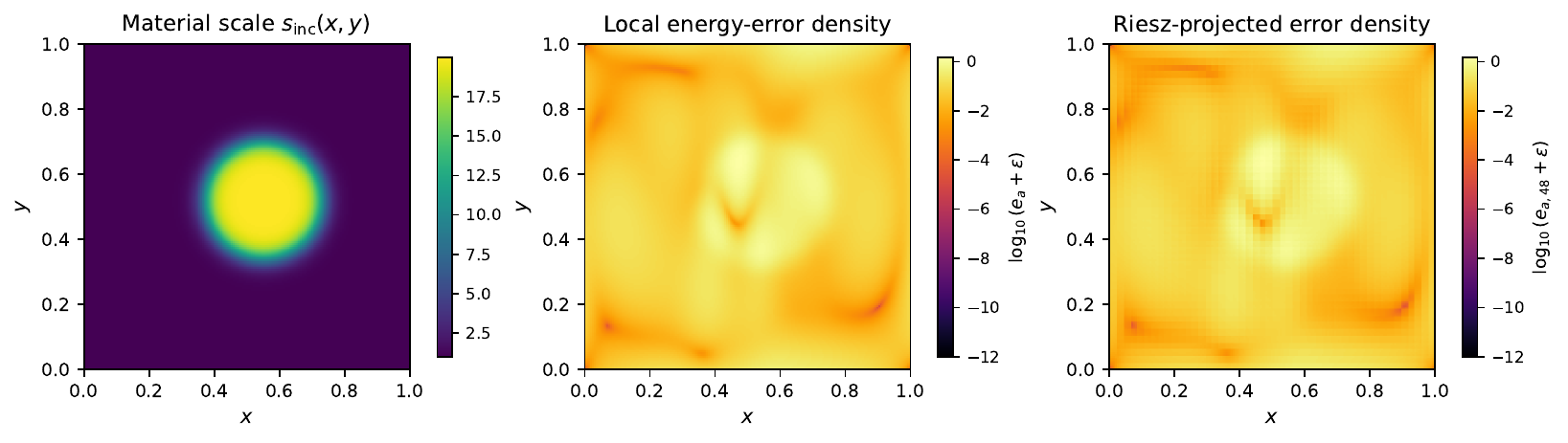}
\caption{Spatial energy diagnostics for the smoothed high-contrast inclusion.
Left: material scale \(s_{\rm inc}(x,y)\). Middle: reference energy-error
density
\(e_a(x)=2\mu|\varepsilon(u_\theta-u^\ast)|^2
+\lambda|\operatorname{div}(u_\theta-u^\ast)|^2\).
Right: conforming Riesz-projected density
\(e_{a,48}(x)=2\mu|\varepsilon(z_{48})|^2
+\lambda|\operatorname{div}z_{48}|^2\).
Both densities are displayed as \(\log_{10}(\,\cdot+\varepsilon)\).}
\label{fig:app_elasticity_inclusion_spatial}
\end{figure}

Figure~\ref{fig:app_elasticity_inclusion_spatial} shows that the
Riesz-projected density identifies the same dominant energetic region
associated with the localized material transition as the reference density.

\begin{figure}[!t]
\centering
\includegraphics[width=\textwidth]
{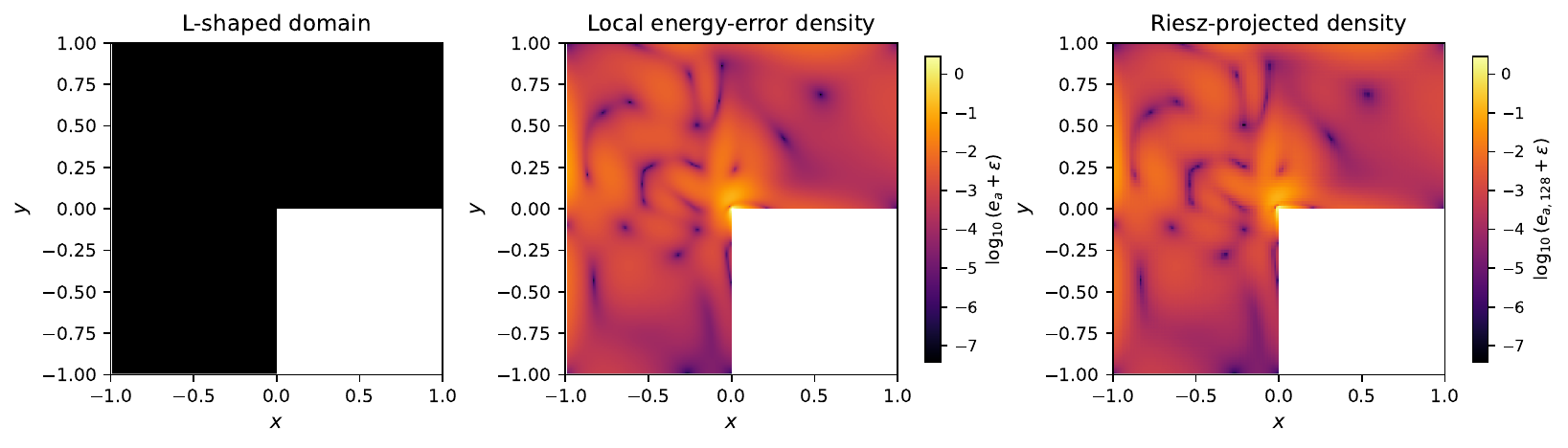}
\caption{Spatial diagnostics for the L-shaped-domain stress test.
Left: domain geometry. Middle: reference energy-error density
\(e_a(x)=|\nabla(u_\theta-u^\ast)(x)|^2\). Right: conforming
Riesz-projected density \(e_{a,128}(x)=|\nabla z_{128}(x)|^2\) on the
\(128\times128\) auxiliary space. Densities are shown as
\(\log_{10}(\,\cdot+\varepsilon)\).}
\label{fig:app_lshape_spatial}
\end{figure}

Figure~\ref{fig:app_lshape_spatial} shows that both densities concentrate
primarily near the reentrant corner while retaining nonzero contributions
away from it, consistent with the observed need for global auxiliary
refinement.

\section{Discrete-hierarchy and quadrature-refinement audit}
\label{app:discrete_quadrature_audit}

All audits below are post-training and use the same trained trajectories and
saved checkpoints as the main experiments.

\subsection{Monitor-hierarchy audit}

For the scalar-diffusion and L-shaped benchmarks, we recompute the reported
monitor diagnostics at the saved operationally selected checkpoints on the
same auxiliary levels using the common finest-level Galerkin restrictions of
Section~\ref{sec:numerical}. No monotonicity violation is observed. Increasing
the quadrature order of the common Galerkin assembly from \(3\) to \(5\)
changes the monitor values by at most \(1.314\times10^{-6}\) for scalar
diffusion and \(6.486\times10^{-4}\) for the L-shaped problem.

These quadrature-refinement diagnostics provide
empirical evidence on the numerical sensitivity
of the evaluated Riesz monitors. In the framework
of Corollary~\ref{cor:numerical_monitor_stability},
a uniform bound on such evaluation errors would
translate directly into a bound on selection regret.
The reported inter-quadrature differences are
observed variations, rather than certified error
bounds relative to the exact conforming monitors,
and therefore do not establish this condition.

\subsection{Validation-oracle stability}

As noted after Proposition~\ref{prop:best_approximation}, a quadrature
approximation can perturb the ordering induced by an exact energy criterion.
For the manufactured benchmarks, \(E_{\rm quad}\) numerically defines the
validation oracle, and we assess whether the reported ordering is stable under
quadrature refinement. Without retraining, the manufactured-reference energies
are reevaluated at tensor-product Gauss orders \(Q=3,5,7\).
Table~\ref{tab:quadrature_refinement_audit} reports the largest checkpointwise
relative changes.

\begin{table}[!hbtp]
\centering
\caption{Manufactured-reference quadrature-refinement audit. The first three
rows concern saved decision-relevant outputs; the \(\kappa_4\) row uses the
complete \(205\)-checkpoint comparison archive.}
\label{tab:quadrature_refinement_audit}
\footnotesize
\setlength{\tabcolsep}{6pt}
\begin{adjustbox}{max width=\textwidth}
\begin{tabular}{lccc}
\toprule
Benchmark
& Audited outputs
& Maximum \(Q=3\to5\)
& Maximum \(Q=5\to7\) \\
\midrule
Scalar diffusion
& \(20\)
& \(1.494\times10^{-7}\)
& \(1.494\times10^{-7}\) \\
L-shaped domain
& \(5\)
& \(2.932\times10^{-3}\)
& \(5.721\times10^{-4}\) \\
Elasticity
& \(16\)
& \(1.043\times10^{-7}\)
& \(1.237\times10^{-7}\) \\
\(\kappa_4\) full archive
& \(205\)
& \(2.263\times10^{-7}\)
& \(4.759\times10^{-8}\) \\
\bottomrule
\end{tabular}
\end{adjustbox}
\end{table}

The reference-energy evaluations are stable under refinement. The largest
sensitivity occurs for the singular L-shaped problem and decreases from
\(Q=3\to5\) to \(Q=5\to7\). On the complete \(\kappa_4\) archive, the oracle
and runner-up remain unchanged in all five trajectories, while the smallest
relative oracle--runner-up gap, \(1.986\times10^{-3}\), remains well above the
observed quadrature variations. The trajectory-oracle inclusion counts are
also unchanged:
\(\mathcal K_{100}:3/5\), \(\mathcal K_{200}:1/5\), and
\(\mathcal K_{400}:0/5\). Thus the validation ordering underlying the reported
\(\kappa_4\) oracle and logging-resolution conclusions is insensitive to the
tested quadrature refinements.

For the perforated plate, refining the independent FEM reference from level
\(192\) to \(256\) likewise preserves the FEM-reference oracle in all five
runs, with maximum relative checkpoint-error change \(6.26\times10^{-4}\) and
rank correlation at least \(0.999826\).

Together, these audits support numerical stability of the reported selection
and logging-resolution conclusions under the tested quadrature and FEM
refinements.

\begin{remark}[Continuous-energy oracle certification]
The refinement audit above establishes numerical stability with respect to the
tested discretizations. Let \(E_k\) denote the exact continuous energy error
and \(\widetilde E_k\) its numerical approximation for \(k\) in a finite
archive \(\mathcal K\). If
\(\max_{k\in\mathcal K}|\widetilde E_k-E_k|\leq\varepsilon\) and the unique
numerical minimizer \(\widetilde k\in\mathcal K\) has oracle--competitor gap
\(\min_{j\in\mathcal K\setminus\{\widetilde k\}}\allowbreak
(\widetilde E_j-\widetilde E_{\widetilde k})>2\varepsilon\),
then \(\widetilde k\) is also the unique continuous-energy oracle, since
\(E_j-E_{\widetilde k}\geq
\widetilde E_j-\widetilde E_{\widetilde k}-2\varepsilon>0\)
for every \(j\in\mathcal K\setminus\{\widetilde k\}\). Thus certification of
an \(E_{\rm quad}\)-based validation oracle against the exact continuous energy
would follow from a suitable uniform numerical-integration error bound; such
bounds are left to future work.
\end{remark}

\subsection{Quadrature sensitivity of archive-level checkpoint ordering}
\label{app:kappa4_ordering_sensitivity}

We assess the sensitivity of the \(\kappa_4\) checkpoint selection itself,
beyond the validation-oracle audit above. Without retraining, all \(41\)
checkpoints in \(\mathcal K_{50}\) are reevaluated for each of five seeds
at tensor-product Gauss orders \(Q\in\{3,5,7\}\), giving \(205\)
distinct checkpoints and \(615\) checkpoint--quadrature evaluations.
For every order we evaluate the conforming Riesz monitor scores
\(\widehat\eta_{24}\), \(\widehat\eta_{48}\), and
\(\widehat\eta_{96}\) using a common finest-level assembled form, and
independently evaluate the Ritz functional on the same quadrature nodes,
denoted \(J_{\rm quad}\). The latter is distinct from the finite-sample
validation loss \(J_{\rm val}\) reported in
Table~\ref{tab:scalar_checkpoint_selection}.
The order-\(3\) Riesz values reproduce the archived monitors with maximum
absolute discrepancies below \(2.9\times10^{-11}\) across the three
auxiliary levels.

For criterion \(c\), seed \(s\), and checkpoint \(k\), let
\(S_{c,s,k}^{(Q)}\) be the numerically evaluated score and let
\(\widehat k_{c,s}^{(Q)}\) minimize it over \(\mathcal K_{50}\). Define
\[
\begin{aligned}
g_c
&:=\min_{s,Q}\ \min_{k\in\mathcal K_{50}\setminus
\{\widehat k_{c,s}^{(Q)}\}}
\bigl(S_{c,s,k}^{(Q)}-S_{c,s,\widehat k_{c,s}^{(Q)}}^{(Q)}\bigr),\\
\Delta_c
&:=\max_{s,\,Q<Q'}\ \max_{k\in\mathcal K_{50}}
\bigl|S_{c,s,k}^{(Q)}-S_{c,s,k}^{(Q')}\bigr|.
\end{aligned}
\]
Thus \(g_c\) is the smallest observed best-versus-runner-up score gap
across the five seeds and three quadrature orders, whereas \(\Delta_c\)
is the largest pointwise interquadrature score difference over the full
archive. Each criterion is compared on its own score scale.

\begin{table}[!hbtp]
\centering
\caption{Full-archive \(\kappa_4\) quadrature-ordering audit: \(41\)
checkpoints per seed, five seeds, and \(Q=3,5,7\). Oracle matches are
reported per quadrature order against the \(Q=7\) reference-energy oracle.
The gap \(g_c\) and difference \(\Delta_c\) are empirical score-scale
quantities, not certified errors relative to exact variational monitors.}
\label{tab:kappa4_quadrature_ordering}
\footnotesize
\setlength{\tabcolsep}{5pt}
\begin{adjustbox}{max width=\textwidth}
\begin{tabular}{lcccc}
\toprule
Criterion & Oracle matches (each \(Q\)) & \(g_c\) & \(\Delta_c\)
& \(g_c/(2\Delta_c)\) \\
\midrule
Riesz \(\widehat\eta_{24}\) & \(5/5\)
& \(2.186\times10^{-4}\) & \(5.269\times10^{-8}\) & \(2074\) \\
Riesz \(\widehat\eta_{48}\) & \(5/5\)
& \(2.317\times10^{-4}\) & \(3.982\times10^{-8}\) & \(2909\) \\
Riesz \(\widehat\eta_{96}\) & \(5/5\)
& \(2.350\times10^{-4}\) & \(4.905\times10^{-8}\) & \(2396\) \\
Ritz \(J_{\rm quad}\) & \(5/5\)
& \(2.811\times10^{-5}\) & \(4.209\times10^{-8}\) & \(334\) \\
\bottomrule
\end{tabular}
\end{adjustbox}
\end{table}

For each seed, every tested quadrature order selects the same checkpoint
under all four criteria; the selected steps for seeds \(0,1,2,3,4\) are,
respectively, \(1550,1400,1900,1100,1450\). All match the
\(Q=7\) energy-reference oracle, yielding \(5/5\) oracle matches at
\emph{each} tested quadrature order. For two numerically evaluated score
vectors, a unique minimizer is preserved whenever its runner-up gap
exceeds twice their maximum pointwise difference. The conservative ratios
\(g_c/(2\Delta_c)>1\) in
Table~\ref{tab:kappa4_quadrature_ordering} verify this sufficient
stability condition for all pairs of tested quadrature orders.

The comparison concerns empirical ordering stability, not certified
continuous-energy selection. In particular, \(\Delta_c\) is an observed
difference between numerical evaluations, not a bound on
\(|\widehat\eta_{m,k}-\eta_{m,k}|\) relative to the exact conforming
monitor. Therefore this audit illustrates the perturbation mechanism of
Corollary~\ref{cor:numerical_monitor_stability} but does not verify its
condition \(\varepsilon_m+2\tau_m<\delta_E\). It also shows that the
quadrature-evaluated Ritz score, like the Riesz monitors, is stable on
this benchmark, rather than establishing general numerical superiority
of one selection rule over the other.

\FloatBarrier

\printbibliography

\end{document}